\documentclass{article}
\usepackage[preprint]{neurips_2026}

\usepackage[utf8]{inputenc}
\usepackage[T1]{fontenc}
\usepackage{amsmath,amssymb,amsthm}
\usepackage{mathtools}
\usepackage{graphicx}
\usepackage{natbib}
\usepackage{hyperref}
\usepackage{url}
\usepackage{booktabs}
\usepackage{enumitem}
\usepackage{xcolor}
\usepackage{caption}
\usepackage{subfig}
\usepackage{float}
\usepackage{placeins}
\usepackage{tabularx}

\usepackage[normalem]{ulem}

\newtheorem{theorem}{Theorem}
\newtheorem{proposition}[theorem]{Proposition}

\newtheorem{corollary}[theorem]{Corollary}

\theoremstyle{definition}
\newtheorem{definition}{Definition}

\newtheorem{remark}{Remark}

\newtheorem{assumption}{Assumption}

\newcommand{\K}{\mathcal{K}}
\newcommand{\X}{\mathcal{X}}
\newcommand{\Kp}{\mathcal{K}_t^+}

\newcommand{\Qt}{Q_t}
\DeclareMathOperator{\supp}{supp}

\title{NOVA: Fundamental Limits of Knowledge Discovery Through AI}

\author{
  Salman Avestimehr\thanks{Corresponding author.}\thanks{Authors are listed alphabetically.} \\
  University of Southern California \\
  \texttt{avestime@usc.edu} \\
  \And
  Ken Duffy \\
  Northeastern University \\
  \texttt{k.duffy@northeastern.edu} \\
  \And
  Muriel M\'{e}dard \\
  Massachusetts Institute of Technology \\
  \texttt{medard@mit.edu}
}

\begin{document}

\maketitle


\begin{abstract}
Can AI systems discover new knowledge through iterative self-improvement, and at what cost? We introduce NOVA, which models the ``generate, verify,
accumulate, retrain'' loop as an adaptive sampling process over a knowledge
space. We give sufficient conditions for accumulated genuine knowledge to
cover a finite domain and show how violations produce contamination,
forgetting, exploration failure, and acceptance failure.

We then analyze how adaptive generation arises from recursive retraining. In
an explicit distribution-level model where accepted artifacts influence the
next generator, we identify a recursive-feedback phase transition. Unanchored feedback can lock generation onto early accepted artifacts and leave initially reachable valid artifacts undiscovered with positive probability. Anchoring updates to a persistent base distribution prevents unbounded distortion and guarantees continued exposure.

Under imperfect verification, we identify a contamination trap: as easy
knowledge is exhausted, even small false-positive rates can admit invalid
artifacts faster than genuine discoveries. We show that Good--Turing
estimation is a local batch-diversity diagnostic, not an estimator of the
historically undiscovered valid mass governing long-term progress. Under a
Zipf tail with exponent \(\alpha>1\), the cumulative generation cost of
obtaining \(D\) distinct genuine discoveries satisfies
\(
R_{\rm cum}(D)=\Theta(c_{\rm gen}D^\alpha).
\)
When the valid base distribution has such a tail, anchored retraining
preserves the exposure needed for this scaling law. Finally, we show how
human guidance, generation, and verification can redirect or expand discovery when autonomous sampling stalls because of repetition, vanishing exposure, or unreliable verification.
\end{abstract}

\section{Introduction}
\label{sec:intro}

Can artificial intelligence discover genuinely new knowledge, or is it fundamentally limited to recombining what it has already seen? Recent AI systems have begun to answer this question empirically: AlphaProof achieved silver-medal performance at the International Mathematical Olympiad by generating and verifying formal proofs~\citep{alphaproof2025}, DeepSeek-Prover-V2 advanced formal mathematical reasoning through reinforcement learning and subgoal decomposition~\citep{deepseekv2_2025}, and self-training methods such as STaR have shown that models can bootstrap their own reasoning capabilities~\citep{zelikman2022star}. These successes share a common architecture: a model generates candidate artifacts, a verifier filters for correctness, and the verified outputs are fed back to improve the model. This ``generate, verify, accumulate, retrain'' loop is emerging as the dominant paradigm for AI-driven knowledge discovery.

Yet, despite these empirical advances, we lack a theoretical understanding of the fundamental limits of this process. How fast can an AI system discover new knowledge, and at what cost? Does the rate of discovery inevitably slow down, and if so, how severely? Under what conditions does the process converge to the target knowledge domain, and when does it collapse? What role does verification quality play, and what happens when verification is imperfect? And perhaps most importantly: can AI systems bootstrap themselves to discover all knowledge autonomously, or is human guidance fundamentally necessary?

\subsection{Overview and Contributions}

This paper introduces the NOVA (\textbf{N}avigating the \textbf{O}rigins and
\textbf{V}erification of \textbf{A}I Knowledge) framework to study when AI
systems can discover genuinely new knowledge through iterative
self-improvement. NOVA models discovery as an adaptive sampling loop over a
knowledge space: a model generates candidates, a verifier accepts or rejects
them, accepted candidates are accumulated, and the model is retrained. The
resulting sampling sequence is endogenous: accepted outputs affect the
retraining state, which changes the future generator. NOVA first develops
results for general adaptive generator sequences and then analyzes an explicit
recursive retraining model to determine when the loop preserves exploration
and when it produces self-reinforcing lock-in. This view separates the core
bottlenecks of AI-driven discovery: reachability, verification, retention, and
the thinning of the unknown frontier. It also connects the problem to species
estimation, occupancy laws, and support-limited sampling, 
yielding convergence guarantees, feedback-induced failure modes, and
discovery-cost laws. Our main contributions are:

\textbf{NOVA formalization and convergence.}
We formalize AI-driven discovery over an ambient candidate space containing both
valid and invalid artifacts. We give sufficient conditions for almost-sure
coverage of a finite knowledge domain and show how violating these conditions
produces distinct failure modes: forgetting, exploration failure, acceptance
failure, and contamination. This separates successful discovery from collapse,
support loss, rejection of valid artifacts, and accumulation of invalid ones.

\textbf{Recursive retraining and endogenous adaptive sampling.}
We analyze an explicit retraining model in which the accepted state reweights
the next generator. Without anchoring, early accepted artifacts can obtain a
persistent advantage that exponentially suppresses undiscovered alternatives,
causing persistent exposure to fail even when the initial generator has full
support. With anchoring to a fixed base distribution, the evolving generator
remains uniformly comparable to that base, guaranteeing continued exploration
and preserving its valid-artifact tail.

\textbf{Imperfect verification and contamination.}
We analyze how genuine and invalid retained artifacts grow under an imperfect
verifier. This yields a local contamination threshold: as the system exhausts
easy discoveries, new genuine artifacts become rarer, so the tolerable
false-positive rate must also shrink. Thus a fixed false-positive rate can become
unsafe near the frontier unless invalid mass shrinks comparably or verification
improves.

\textbf{Missing mass and discovery-cost scaling.}
We distinguish Good--Turing batch unseen mass from the historically
undiscovered valid mass that drives local progress. For cumulative rates, we
first count all accepted valid generation opportunities, including
repetitions, and then apply occupancy to determine how many distinct artifacts
they reveal. Under a Zipf cumulative exposure tail with exponent
\(\alpha>1\), the expected cumulative generation cost of discovering \(D\)
distinct genuine artifacts satisfies
\[
R_{\rm cum}(D)=\Theta(c_{\rm gen}D^\alpha).
\]
For the explicit recursive retraining model, anchored feedback guarantees the
required exposure and tail conditions whenever the valid portion of the base
distribution has the corresponding Zipf tail.

\textbf{Human amplification.}
We formalize how human experts amplify NOVA through guidance, generation, and
verification. Human input can increase the mass assigned to new valid artifacts,
improve acceptance of valid candidates, add expert-generated candidates, and
expand the reachable support. This explains why human input is especially valuable near autonomous
exploration barriers, where the generator may assign vanishing probability to
new valid artifacts or recursive retraining may reinforce a narrow region of
the search space.


\subsection{Related Work}
NOVA sits at the intersection of recursive synthetic-data training, verified AI
reasoning, self-improvement, and classical missing-mass estimation. These areas
study different pieces of the generate--verify--accumulate--retrain loop. Our goal
is to formalize the loop as a unified discovery process and characterize when it
converges, when it fails, and how its cost scales.

\textbf{Model collapse.}
A growing line of work studies the risks of recursive training on model-generated
data. \citet{shumailov2024} showed that training on recursively generated synthetic
data can cause model collapse. \citet{gerstgrasser2024} showed that accumulating
real data alongside synthetic data can mitigate this effect, while
\citet{seddik2024} provided a statistical analysis of collapse mechanisms. 
These works focus on distributional degradation under recursive training.
NOVA studies the corresponding discovery consequences: whether recursive
retraining preserves exposure to undiscovered valid artifacts, whether it
locks the generator onto early accepted outputs, and whether verifier errors
are amplified through future generation.


\textbf{AI theorem proving.}
AlphaProof achieved IMO silver-medal performance using Lean verification
\citep{alphaproof2025}, while DeepSeek-Prover-V1.5 reached 63.5\% on miniF2F-test
\citep{deepseekv15_2025}, and DeepSeek-Prover-V2 advanced formal reasoning through
reinforcement learning and subgoal decomposition \citep{deepseekv2_2025}. These
systems instantiate NOVA in a near-perfect-verification regime, where false
positives are mechanically controlled.

\textbf{Self-training and self-play.}
STaR bootstraps reasoning \citep{zelikman2022star}, ReST applies reinforced
self-training \citep{gulcehre2023rest}, and AlphaGo Zero demonstrated
superhuman performance through pure self-play \citep{silver2017}. NOVA models
such generate--filter--retrain systems as adaptive sampling processes and
introduces an explicit recursive retraining model to analyze how accepted
outputs shape the resulting adaptive sequence. Its lock-in theorem shows how
reinforcement of early accepted outputs can destroy persistent exposure, while
its anchoring theorem gives a sufficient condition preventing recursive
narrowing.

\textbf{Species estimation.}
Good--Turing estimation \citep{good1953} and its extensions
\citep{efron1976,orlitsky2016,mcallester2000,Painsky2021,painsky2022,painsky2023,
lee2025,han2025bestinggoodturingoptimalitynonparametric}, including results for
specific distributional models
\citep{chandra2023goodgoodturingmarkovsamples,wolfer2020statisticalestimationergodicmarkov},
provide the classical foundation for missing-mass estimation. Error and convergence
analyses are closely related
\citep{pal2026blindspotmassgoodturingframework,Skorski2021,Chandra2024,Cohen22,
Acharya2018}. NOVA uses these tools as local diagnostics of diversity within a fixed
generation batch. Cumulative discovery, however, depends on the
repeat-inclusive accepted valid exposure produced across a changing sequence
of generators. Our general scaling law therefore imposes a tail condition on
the cumulative exposure distribution rather than on a single batch or only
the currently undiscovered frontier. For the anchored recursive retraining
model, this condition follows from the tail of the persistent base
distribution.



\section{Problem Formalization: The NOVA Framework}
\label{sec:formalization}
\label{def:knowledge-space}%

Let \(\K\) be the set of valid knowledge artifacts, and let
\(\X \supseteq \K\) be the ambient candidate space, which may also contain
invalid candidates. The ideal knowledge distribution \(P\) is a distribution
over \(\K\) representing the intrinsic difficulty of discovering valid artifacts:
artifacts with larger \(P(k)\) are easier to encounter under an ideal generator.

At iteration \(t\), the model \(\mathcal M_t\) has a retained set
\(\widehat{\K}_t \subseteq \X\) of accepted candidates accumulated from previous
iterations. Its genuine component is $\K_t^+ = \widehat{\K}_t \cap \K$,
and its retained invalid component is \(\widehat{\K}_t\setminus \K\). The model
\(\mathcal M_t\) induces an actual sampling distribution \(Q_t\) over  \(\X\). The distinction between \(Q_t\) and \(P\) is important: \(Q_t\) determines what the
model actually generates, while \(P\) is used later to characterize idealized
difficulty and tail behavior.

For the explicit recursive retraining model developed below, we additionally
introduce a fixed base distribution \(B\) over \(\X\). The distribution \(B\)
represents persistent generative support supplied, for example, by the
pretrained model, replay of original data, explicit exploration, mutation, or
another component that is not overwritten by recursive retraining. We assume
that \(B(x)>0\) on the region of interest. The distributions \(P\) and \(B\)
serve distinct roles: \(P\) is a domain-level reference for the intrinsic
difficulty of valid artifacts, whereas \(B\) is an operational component of
the retraining dynamics that preserves exploratory support across iterations.

Finally, let \(\mathcal F_t\) denote the \(\sigma\)-algebra generated by the history
before generation at iteration \(t\), including
\(\widehat{\K}_t\), \(\K_t^+\), \(Q_t\), and all randomness from previous
iterations.

\paragraph{The NOVA Loop.}
At each iteration \(t=0,1,2,\ldots\), NOVA executes the following steps.
\begin{enumerate}[leftmargin=*,itemsep=1pt]
\item \textbf{Generate:} The model \(\mathcal M_t\) generates \(N\) artifacts $c_{t,1},\ldots,c_{t,N}\in \X$
$\mathrm{i.i.d.}$ from \(Q_t\).

\item \textbf{Verify:} Apply a verifier \(V\) to each candidate
\(c_{t,i}\). For every valid artifact \(k\in\K\), define
\[
r_{t,k}
=
\Pr\!\left[
V(k)=1
\mid
k\text{ is generated at iteration }t,\mathcal F_t
\right].
\]
Thus \(r_{t,k}\) is the conditional probability that a generated occurrence
of \(k\) is accepted. When this probability is uniform over a specified set
of valid artifacts, we write \(r_{t,k}=r_t\). For a generic invalid candidate
\(C_t\sim Q_t\), define
\[
\delta_t
=
\Pr\!\left[
V(C_t)=1
\mid
C_t\in\X\setminus\K,\mathcal F_t
\right].
\]

\item \textbf{Accumulate:} Let \(\mathcal A_t=\{c_{t,i}:V(c_{t,i})=1\}\) be the accepted candidates, and update the retained set as \(\widehat{\K}_{t+1}=\widehat{\K}_t\cup\mathcal A_t\), with updated genuine component \(\K_{t+1}^+=\widehat{\K}_{t+1}\cap\K\).

\item \textbf{Retrain:} The model is updated using
\(\widehat{\K}_{t+1}\), producing the next sampling distribution
\[
Q_{t+1}
=
F_t(Q_t,\widehat{\K}_{t+1}).
\]
The general NOVA framework permits an arbitrary history-dependent retraining
map \(F_t\), and hence an adaptive sequence of generator distributions
\(\{Q_t\}_{t\ge0}\). Below we study an explicit recursive retraining model in
which the accepted state directly reweights future generation.
\end{enumerate}

\paragraph{A canonical recursive retraining model.}
We now introduce a tractable distribution-level model of how retained outputs
alter future generation. The model is not intended to represent the full
parameter-space dynamics of neural-network training. Rather, it isolates the
competition between reinforcement from accepted artifacts and a persistent
source of exploratory support.

Let
\[
s_t(x)
=
s(x;\widehat{\K}_{t+1})
\]
denote the retraining score assigned to artifact \(x\) by the retained state
after iteration \(t\). A simple example is direct archive reinforcement,
\[
s_t(x)
=
\mathbf 1\{x\in\widehat{\K}_{t+1}\},
\]
under which retained artifacts receive a positive score. More general score
functions may reward artifacts that are similar to, derived from, or judged
useful relative to retained outputs. Thus accepting one artifact can alter
the future probability not only of that exact artifact, but also of a broader
region of the candidate space.

We define \(Q_{t+1}\) through the regularized optimization problem
\[
Q_{t+1}
=
\arg\max_{Q\in\Delta(\X)}
\left\{
\eta\,\mathbb E_Q[s_t(X)]
-
D_{\rm KL}(Q\Vert Q_t)
-
\lambda D_{\rm KL}(Q\Vert B)
\right\},
\]
where \(\Delta(\X)\) is the set of probability distributions over \(\X\).
The three terms have distinct roles. The score term favors artifacts rewarded
by the newly retained state. The divergence from \(Q_t\) makes the update
conservative relative to the current generator. The divergence from \(B\)
anchors the update to a persistent reference distribution that is not itself
overwritten by recursive feedback. The parameter \(\eta>0\) controls the
strength of accepted-state reinforcement, while \(\lambda\ge0\) controls the
strength of anchoring.

Solving the optimization problem gives
\[
Q_{t+1}(x)
=
\frac{
Q_t(x)^\rho
B(x)^{1-\rho}
\exp\{\gamma s_t(x)\}
}{
Z_t
},
\]
where
\[
\rho
=
\frac{1}{1+\lambda},
\qquad
\gamma
=
\frac{\eta}{1+\lambda},
\]
and \(Z_t\) is the normalizing constant. Before the score tilt is applied,
the update forms a geometric interpolation between the current generator
\(Q_t\) and the persistent base distribution \(B\). The coefficient
\(\rho\) controls how strongly the next generator inherits the current
generator, while \(\gamma\) is the effective strength of the new score.

When \(\lambda=0\), we have \(\rho=1\) and \(\gamma=\eta\), so
\[
Q_{t+1}(x)
=
\frac{Q_t(x)e^{\eta s_t(x)}}{Z_t}.
\]
We call this \emph{unanchored recursive feedback}. For any two artifacts
\(x\) and \(y\) with positive probability,
\[
\log\frac{Q_{t+1}(x)}{Q_{t+1}(y)}
=
\log\frac{Q_t(x)}{Q_t(y)}
+
\eta\bigl(s_t(x)-s_t(y)\bigr).
\]
Thus score differences accumulate additively in log-probability ratios. If
an early accepted artifact repeatedly receives a higher score than an
undiscovered alternative, its relative advantage compounds across
iterations, producing exponential suppression of the alternative. There is
no persistent restoring force that pulls the generator back toward a broader
reference distribution. This cumulative log-odds effect is the basic
mechanism of recursive lock-in analyzed in the next section.

When \(\lambda>0\), we have \(0<\rho<1\). We call this
\emph{anchored recursive feedback}. To see the effect of anchoring, define
the log-density ratio relative to \(B\) by
\[
h_t(x)
=
\log\frac{Q_t(x)}{B(x)}
\]
on the common support of \(Q_0\) and \(B\). For any two artifacts \(x\) and
\(y\) in this support,
\[
h_{t+1}(x)-h_{t+1}(y)
=
\rho\bigl(h_t(x)-h_t(y)\bigr)
+
\gamma\bigl(s_t(x)-s_t(y)\bigr).
\]
The factor \(\rho<1\) contracts previously accumulated distortion relative
to \(B\), while the score term introduces new distortion. Anchoring therefore
does not prevent the model from learning from accepted artifacts; it prevents
past score advantages from accumulating without bound.

Increasing \(\lambda\) decreases both \(\rho\) and \(\gamma\), strengthening
the restoring effect of the base distribution and weakening each individual
score update. Increasing \(\eta\) strengthens the score tilt. Under a uniform
bound \(L\) on score oscillation, the resulting long-run distortion scale is
controlled by the ratio
\[
\frac{\eta L}{\lambda}.
\]
Thus the relative strength of reinforcement and anchoring, rather than either
parameter alone, determines how far the evolving generator can deviate from
the persistent base distribution.

Because the update uses geometric rather than additive interpolation,
anchoring does not introduce support absent from the current generator or
from \(B\). Its role is to preserve controlled relative exposure within
their common support. Support expansion requires a separate mechanism, such
as mutation, composition, explicit exploration with broader support, or
human guidance.

This recursive retraining model is not required by the general NOVA
definitions or by the mechanism-agnostic results for discovery, coverage,
and contamination. It provides an explicit endogenous model for the adaptive
sequence \(\{Q_t\}_{t\ge0}\), allowing us to determine when recursive
retraining destroys persistent exposure, when anchoring preserves it, and
when the tail structure of the base distribution is transferred to
cumulative discovery rates.

Appendix~\ref{app:examples} illustrates these components through three motivating settings: formal proof discovery, molecular discovery, and scientific hypothesis generation.

\paragraph{Interpretation.}
NOVA separates several objects that are often conflated. The set \(\K\) is
the target knowledge domain: the valid artifacts that could in principle be
discovered. The distribution \(P\) describes an idealized domain-level
difficulty profile, assigning larger mass to artifacts that are intrinsically
easier to encounter. The adaptive distribution \(Q_t\) describes what the
current AI system actually generates. For the explicit recursive retraining
model, \(B\) provides persistent operational support against which the
evolving generator is anchored. Thus discovery is governed by the interaction
among validity, adaptive generation, verification, retained outputs, and the
retraining mechanism.

\paragraph{Key quantities.}
\label{def:batch-mass}%
\label{def:mass-partition}%
\label{rem:mass-distinction}%
The central quantity for discovery is the model mass assigned to historically
undiscovered valid artifacts:
\[
M_t^{\rm new}=\sum_{k\in K\setminus K_t^+} Q_t(k).
\]
When \(M_t^{\rm new}\) is large, the current model frequently generates new valid
artifacts; when it is small, generation mostly repeats already discovered artifacts
or produces invalid ones. We decompose the remaining probability mass as
\[
A_t=\sum_{k\in K_t^+}Q_t(k),
\qquad
U_t=\sum_{x\in \mathcal X\setminus K}Q_t(x),
\]
so that \(M_t^{\rm new}+A_t+U_t=1\). Here \(A_t\) is the mass on already discovered
genuine artifacts, while \(U_t\) is the mass on invalid candidates.

We also define the total valid mass
\[
M_t^{\mathrm{val}}
=
Q_t(\K)
=
\sum_{k\in\K}Q_t(k)
=
A_t+M_t^{\mathrm{new}}
=
1-U_t.
\]
Unlike \(M_t^{\mathrm{new}}\), the quantity
\(M_t^{\mathrm{val}}\) includes probability assigned to both previously
discovered and still-undiscovered valid artifacts.

Let
\[
S_t =
\{k \in \K \setminus \Kp :
k \text{ is generated at least once at iteration } t
\text{ and accepted}\}.
\]

Assume that verifier decisions are conditionally independent across generated
candidates. Then
\[
\mathbb E[|S_t|\mid\mathcal F_t]
=
\sum_{k\in\K\setminus\K_t^+}
\left[
1-\left(1-r_{t,k}Q_t(k)\right)^N
\right].
\]
Under the uniform specialization of true-positive rate \(r_{t,k}=r_t\) over the undiscovered valid
region,
\[
\begin{aligned}
\mathbb E[|S_t|\mid\mathcal F_t]
&=
\sum_{k\in\K\setminus\K_t^+}
\left[
1-\left(1-r_tQ_t(k)\right)^N
\right]\\
&=
Nr_tM_t^{\rm new}
-
\binom{N}{2}r_t^2
\sum_{k\in\K\setminus\K_t^+}Q_t(k)^2
+\cdots .
\end{aligned}
\]
We refer to the setting \(Nr_tQ_t(k)\ll1\) for all relevant
\(k\in\K\setminus\K_t^+\) as the \textbf{sparse regime}. In this regime,
duplicate discoveries within a batch are negligible, and
\[
\mathbb E[|S_t|\mid\mathcal F_t]
\approx
Nr_tM_t^{\rm new}.
\]

A separate notion of missing mass arises within a single generation batch. Given
\(X_1,\ldots,X_N\sim Q_t\), define the ambient batch unseen mass
\[
M^{\rm batch}_{t,\X}
=
\sum_{x\in\X}Q_t(x)\mathbf 1[x\notin\{X_1,\ldots,X_N\}].
\]
Good--Turing~\citep{good1953} estimates this local batch-unseen mass, which
diagnoses repetition or loss of diversity under the current generator. It is distinct
from \(M_t^{\rm new}\), the historically undiscovered valid mass that drives
cumulative discovery; Appendix~\ref{app:missing-mass} formalizes the distinction
and recalls Good--Toulmin forecasting~\citep{good1956}.

\section{Coverage, Collapse, and Imperfect Verification}
\label{sec:convergence}

We now study when NOVA succeeds or fails by giving conditions for almost-sure coverage, identifying exploration barriers, and analyzing contamination under imperfect verification. 

\subsection{Sufficient Conditions for Convergence}

Recall that \(\mathcal F_t\) denotes the history before generation at
iteration \(t\). In this section, it includes
\(\K_t^+\), \(\widehat{\K}_t\), \(Q_t\), the artifact-wise acceptance
probabilities \(\{r_{t,k}:k\in\K\}\), \(\delta_t\), and all randomness from
previous iterations.

\begin{theorem}[Sufficient Conditions for Almost-Sure Coverage]
\label{thm:convergence}
Assume that \(Q_t\) and \(\{r_{t,k}:k\in\K\}\) are
\(\mathcal F_t\)-measurable, and conditional on $\mathcal F_t$, the next batch is sampled i.i.d.\ from $Q_t$. Suppose $|\K| < \infty$, the initial retained state is uncontaminated (i.e. $\widehat{\K}_0\subseteq\K$), 
and the following hold:
\begin{enumerate}[label=\textbf{C\arabic*},itemsep=1pt]
\item \textbf{Monotone accumulation:} $\Kp \subseteq \mathcal{K}_{t+1}^+ $ for all $t$.
\item \textbf{Persistent pre-discovery exposure:}
For each $k \in \K$, on any sample path where $k$ is never discovered,
\[
\sum_{t:\, k \notin K_t^+}
\left(1-(1-Q_t(k))^N\right)=\infty .
\]
\item \textbf{Artifact-wise nondegenerate acceptance:} There exists \(r_{\min}>0\) such that,
for every \(k\in K\), on any sample path where \(k\notin K_t^+\),
\(
r_{t,k}\ge r_{\min}
\).
\item \textbf{No false positives:} $\delta_t = 0$ for all $t$.
\end{enumerate}

Then $\widehat{\K}_t=\K_t^+\to\K$ almost surely.
\end{theorem}

The proof is given in Appendix~\ref{app:convergence-proofs}; the main idea is
an artifact-wise recurrence argument. C1 prevents forgetting of already
discovered genuine artifacts. If \(k\notin\K_t^+\), then, conditional on the
history, the probability that the first candidate in the batch equals \(k\)
and is accepted is at least \(r_{\min}Q_t(k)\). Moreover,
\[
1-\left(1-Q_t(k)\right)^N
\le
NQ_t(k).
\]
Thus C2 implies that
\[
\sum_t Q_t(k)=\infty
\]
on every sample path on which \(k\) remains undiscovered. Together with C3,
this gives divergent cumulative conditional probability of discovering
\(k\). C4 rules out false positives but does not require \(r_{t,k}=1\):
false negatives can slow discovery, whereas false positives corrupt the
retained knowledge base. Finally, the assumption \(|\K|<\infty\) makes
artifact-wise discovery imply eventual full coverage.
Appendix~\ref{app:infinite} discusses the infinite-domain case.

\begin{remark}[Failure modes]
\label{rem:failure-modes}
The conditions identify four failure modes: failure of C1 causes forgetting;
failure of C2 causes exploration failure through support collapse or
distributional narrowing; failure of C3 causes exposed valid artifacts to be
rejected too often; and failure of C4 causes contamination by false
positives.
\end{remark}

Theorem~\ref{thm:convergence} is qualitative: it does not imply convergence of
\(Q_t\) to any fixed distribution or give a discovery rate. It only states that finite-domain coverage occurs almost surely under persistent pre-discovery exposure and retention of accepted genuine artifacts. The next corollary makes explicit the corresponding exploration barrier.

\begin{corollary}[Exploration Barrier]
\label{cor:barrier}
Let $\K_\infty^+ = \lim_{t\to\infty} \K_t^+$ denote the asymptotic knowledge base. Assume retraining is support-preserving:
\[
\operatorname{supp}(Q_{t+1}) \subseteq \operatorname{supp}(Q_t)
\quad \text{for all } t.
\]
Then
\[
\K_\infty^+ \subseteq \operatorname{supp}(Q_0)\cap \K .
\]
\end{corollary}

The corollary follows by induction: support preservation gives
\(\operatorname{supp}(Q_t)\subseteq \operatorname{supp}(Q_0)\) for all \(t\). Hence
any \(k\notin\operatorname{supp}(Q_0)\) is never generated and cannot enter
\(\K_t^+\). Autonomous NOVA therefore cannot discover valid artifacts outside the
initial generative support unless retraining, composition, human guidance, or
another mechanism expands support. For neural generators with broad literal
support, the relevant notion is effective support: artifacts that can be generated
with non-negligible probability under the available compute budget.

Together, Theorem~\ref{thm:convergence} and
Corollary~\ref{cor:barrier} separate two requirements that are easy to
conflate. Theorem~\ref{thm:convergence} says that coverage is possible when
every valid artifact receives persistent pre-discovery exposure and accepted
discoveries are retained. Corollary~\ref{cor:barrier} says that such exposure
cannot arise for artifacts outside the model's initial support unless the
support itself expands. In this sense, autonomous discovery is limited not
only by verification, but also by the geometry of the generator's reachable
support.

\subsection{Recursive Retraining and Persistent Exposure}
\label{sec:recursive-exposure}

Theorem~\ref{thm:convergence} isolates persistent pre-discovery exposure as
the central exploration requirement, but leaves open whether a concrete
retraining mechanism satisfies it. We now answer this question for the
canonical recursive retraining model introduced in
Section~\ref{sec:formalization}. The result shows that recursive reinforcement
can either destroy persistent exposure or preserve it, depending on whether
retraining remains anchored to a persistent base distribution.

For a function \(f:\X\to\mathbb R\), define
\[
\operatorname{osc}(f)
=
\sup_{x\in\X}f(x)-\inf_{x\in\X}f(x).
\]

\begin{theorem}[Recursive Retraining: Lock-In versus Anchoring]
\label{thm:feedback}
Consider the canonical recursive retraining model
\[
Q_{t+1}(x)
=
\frac{
Q_t(x)^\rho
B(x)^{1-\rho}
\exp\{\gamma s_t(x)\}
}{
Z_t
}.
\]

\textbf{(i) Unanchored retraining.}
Suppose \(\lambda=0\), so that
\[
Q_{t+1}(x)
=
\frac{Q_t(x)e^{\eta s_t(x)}}{Z_t}.
\]
Suppose that, beginning at some iteration \(t_0\), an accepted artifact \(i\)
and an undiscovered valid artifact \(k\) satisfy
\[
Q_{t_0}(i)>0,
\qquad
Q_{t_0}(k)>0,
\]
and
\[
s_t(i)-s_t(k)
\ge
\Delta>0
\]
at every iteration while \(k\) remains undiscovered. Then
\[
\frac{Q_t(k)}{Q_t(i)}
\le
\frac{Q_{t_0}(k)}{Q_{t_0}(i)}
e^{-\eta\Delta(t-t_0)}.
\]
Consequently,
\[
\sum_{t=t_0}^{\infty}NQ_t(k)
<
\infty
\]
on the path where \(k\) remains undiscovered, and \(k\) remains undiscovered
with positive probability.

\textbf{(ii) Anchored retraining.}
Suppose \(\lambda>0\), and suppose that \(Q_0\) and \(B\) have the same
support. Assume
\[
\sup_t\operatorname{osc}(s_t)\le L
\]
and
\[
C_0
=
\operatorname{osc}
\left(
\log\frac{Q_0}{B}
\right)
<
\infty,
\]
where the oscillation and the ratio are evaluated over their common support.
Then, for every \(t\) and every \(x\) in this support,
\[
e^{-C}B(x)
\le
Q_t(x)
\le
e^CB(x),
\]
where
\[
C
=
\max\left\{
C_0,\frac{\eta L}{\lambda}
\right\}.
\]
Thus anchored recursive retraining prevents accepted-state reinforcement from
causing unbounded distortion relative to the persistent base distribution.
\end{theorem}

The proof is given in Appendix~\ref{app:feedback-proofs}.

\begin{remark}[Role of the recursive-feedback theorem]
\label{rem:feedback-role}
Theorem~\ref{thm:feedback} is primarily a mechanism theorem for condition C2
of Theorem~\ref{thm:convergence}, rather than a separate coverage theorem.

Under unanchored retraining, part~\textup{(i)} gives
\[
\sum_{t=t_0}^{\infty}NQ_t(k)<\infty
\]
on a positive-probability path on which \(k\) remains undiscovered. Since
\[
1-\left(1-Q_t(k)\right)^N
\le
NQ_t(k),
\]
it follows on that path that
\[
\sum_{t=t_0}^{\infty}
\left[
1-\left(1-Q_t(k)\right)^N
\right]
<
\infty.
\]
Thus part~\textup{(i)} gives an explicit recursive-retraining mechanism by
which the persistent pre-discovery exposure condition C2 fails, even though
\(k\) initially has positive probability.

Under anchored retraining, part~\textup{(ii)} gives
\[
Q_t(k)\ge e^{-C}B(k).
\]
Therefore, for every valid artifact satisfying \(B(k)>0\), every
pre-discovery iteration contributes at least
\[
1-\left(1-e^{-C}B(k)\right)^N>0
\]
to the exposure sum. On any path on which \(k\) remains undiscovered, the sum
therefore diverges, establishing C2.

Theorem~\ref{thm:feedback} does not by itself establish C1, C3, or C4 of
Theorem~\ref{thm:convergence}. Moreover, anchoring preserves exposure only
within the support of \(B\); it does not expand that support.
\end{remark}

\begin{corollary}[Almost-Sure Coverage under Anchored Retraining]
\label{cor:anchored-coverage}
Assume the anchored setting of Theorem~\ref{thm:feedback}, an uncontaminated
initial retained state
\[
\widehat{\K}_0\subseteq\K,
\]
monotone accumulation, and no false positives. Suppose also that
\[
r_{t,k}\ge r_{\min}>0
\]
for every undiscovered valid artifact. If
\[
B(k)>0
\qquad
\text{for every }k\in\K,
\]
then every valid artifact receives persistent pre-discovery exposure.
Consequently, when \(|\K|<\infty\),
\[
\widehat{\K}_t=\K_t^+\to\K
\qquad\text{almost surely}.
\]
\end{corollary}

The proof is given in Appendix~\ref{app:feedback-proofs}.

\begin{remark}[Implications and limits of anchored coverage]
\label{rem:anchored-coverage}
Corollary~\ref{cor:anchored-coverage} closes the logical loop between the
abstract coverage conditions and the explicit recursive-retraining dynamics.
Theorem~\ref{thm:feedback}\textup{(ii)} supplies C2 of
Theorem~\ref{thm:convergence}, while monotone accumulation, nondegenerate
acceptance, and the absence of false positives supply C1, C3, and C4,
respectively.

The conclusion is qualitative and asymptotic. It guarantees eventual
almost-sure coverage of a finite domain, but it does not imply a practical
discovery time or a particular discovery rate. If \(B(k)\) is extremely
small for some artifact, discovering that artifact may still require an
enormous number of samples.

The corollary also does not overcome the support barrier in
Corollary~\ref{cor:barrier}. Anchoring prevents the probability of a
base-supported artifact from vanishing, but it cannot make an artifact
reachable when \(B(k)=0\).
\end{remark}

The preceding results determine when the canonical recursive-retraining model
satisfies or violates the exploration condition C2. They continue to assume
the no-false-positive condition C4. We next examine what happens when C4 is
relaxed and invalid artifacts can enter the retained state.

\subsection{Imperfect Verification and the Contamination Trap}
\label{sec:imperfectVerification}

When verification is imperfect, the retained set may contain both genuine
and invalid artifacts. Let
\[
G_t
=
|\K_t^+|
\]
denote the number of retained genuine artifacts, and let \(I_t\) denote the
cumulative number of invalid accepted candidates, counted with multiplicity.
Define the one-step increments
\[
\Delta G_t
=
G_{t+1}-G_t,
\qquad
\Delta I_t
=
I_{t+1}-I_t.
\]
We compare these increments under true-positive rate \(r_t\) and
false-positive rate \(\delta_t\).

\begin{proposition}[One-step contamination increments]
\label{thm:contamination}
Assume that, conditional on $\mathcal F_t$, the $N$ candidates at iteration $t$
are drawn independently from $Q_t$, and that verifier decisions are conditionally
independent given each candidate. Let $r_{t,k}$ denote the conditional acceptance
probability of an undiscovered valid artifact $k\in K\setminus K_t^+$. Then the
expected number of newly accepted genuine artifacts is
\[
\mathbb E[\Delta G_t\mid \mathcal F_t]
=
\sum_{k\in K\setminus K_t^+}
\left(1-(1-r_{t,k}Q_t(k))^N\right).
\]
If invalid false positives are counted per accepted candidate, then
\[
\mathbb E[\Delta I_t\mid \mathcal F_t]
=
N\delta_t U_t.
\]
In the sparse regime, where $NQ_t(k)\ll1$ for relevant undiscovered artifacts $k$,
and when $r_{t,k}=r_t$ over this region,
\(
\mathbb E[\Delta G_t\mid \mathcal F_t]
\approx
N r_t M_t^{\mathrm{new}}
\),
and hence
\[
\frac{\mathbb E[\Delta I_t\mid \mathcal F_t]}
{\mathbb E[\Delta G_t\mid \mathcal F_t]}
\approx
\frac{\delta_t U_t}{r_t M_t^{\mathrm{new}}}.
\]
\end{proposition}

The proof is in Appendix~\ref{app:contamination-proofs}. Here \(\Delta I_t\) counts invalid false positives per accepted candidate; deduplication only changes the exact finite-batch expression, as detailed in the appendix.

\begin{corollary}[Local contamination threshold]
\label{cor:trap}
Define the marginal contamination fraction
\[
f_t^{\mathrm{marg}}
=
\frac{\mathbb E[\Delta I_t\mid \mathcal F_t]}
{\mathbb E[\Delta G_t\mid \mathcal F_t]+\mathbb E[\Delta I_t\mid \mathcal F_t]}.
\]
In the sparse regime,
\[
f_t^{\mathrm{marg}}
\approx
\frac{\delta_t U_t}
{r_t M_t^{\mathrm{new}}+\delta_t U_t}.
\]
Therefore, to keep $f_t^{\mathrm{marg}}\le f_{\mathrm{critical}}$, it is sufficient that
\[
\delta_t
\le
\delta_t^*
=
\frac{r_tM_t^{\mathrm{new}}f_{\mathrm{critical}}}
{U_t(1-f_{\mathrm{critical}})}.
\]
\end{corollary}

Proposition~\ref{thm:contamination} and Corollary~\ref{cor:trap} identify a
verification bottleneck at the discovery frontier. As \(M_t^{\mathrm{new}}\to 0\), genuine
discoveries become increasingly rare. If the false-positive rate \(\delta_t\)
remains fixed, then invalid candidates can enter the retained set faster than
new genuine artifacts. This is the contamination trap: progress reduces the
new-valid mass, which in turn tightens the verification requirement. In
particular, the safe false-positive threshold satisfies
\[
\delta_t^*
=
\frac{r_tM_t^{\mathrm{new}}f_{\mathrm{critical}}}
{U_t(1-f_{\mathrm{critical}})},
\]
so maintaining a bounded marginal contamination fraction requires
\(\delta_t=O(M_t^{\mathrm{new}}/U_t)\). Thus, unless the invalid mass \(U_t\)
shrinks comparably, verification must become increasingly precise, exactly
when genuine discoveries are hardest to find. Figure~\ref{fig:contamination}
illustrates this effect. 

\begin{remark}[Recursive amplification of verifier errors]
\label{rem:recursive-error-amplification}
Proposition~\ref{thm:contamination} describes how false positives enter the
retained state, while Theorem~\ref{thm:feedback} describes how the retained
state can reshape future generation. These two effects can compound.

If an invalid artifact is falsely accepted and subsequently receives a
persistent score advantage, recursive retraining can increase its future
generation probability. A one-step verification error can therefore affect
not only the current archive but also the future sampling distribution,
causing the error to be repeatedly regenerated and reinforced.

Anchoring limits the resulting distributional distortion relative to \(B\),
but it does not remove invalid retained artifacts or make false positives
harmless. Contamination control and anchoring address distinct requirements:
the former controls what enters the retained state, while the latter controls
how strongly the retained state can distort future exploration.
\end{remark}

Appendix~\ref{app:verification} extends the contamination analysis by adding 
verification cost. It shows how finite per-iteration budgets constrain the feasible
batch size, how discovery becomes verification-limited when checking candidates is much
more expensive than generating them, and why reducing false positives requires
increasing verification effort  near the discovery frontier.


\begin{figure}[t]
\centering
\begin{minipage}[c]{0.45\linewidth}
\caption{Local contamination trap. The curves show the fraction of newly accepted
artifacts that are invalid,
\(f_t^{\mathrm{marg}}\approx \delta_tU_t/(r_tM_t^{\mathrm{new}}+\delta_tU_t)\), as a
function of the false-positive rate \(\delta_t\). As the new-valid mass
\(M_t^{\mathrm{new}}\) decreases, even small false-positive rates can make invalid
artifacts dominate the accepted increments.}
\label{fig:contamination}
\end{minipage}
\hfill
\begin{minipage}[c]{0.52\linewidth}
\centering
\includegraphics[width=\linewidth]{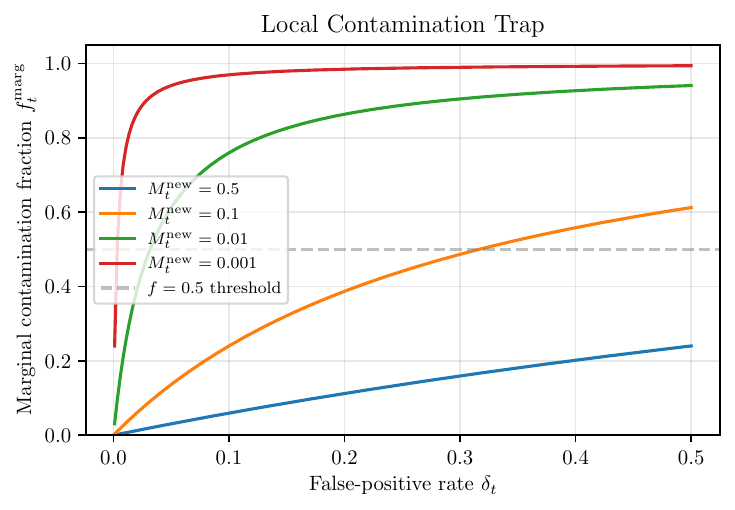}
\end{minipage}
\end{figure}

\section{Discovery Rate and Cost Scaling}
\label{sec:discovery-cost}

Section~\ref{sec:convergence} gave qualitative conditions for coverage and
failure. We now ask a quantitative question: how many samples are required
to obtain \(D\) genuine discoveries? The analysis separates two
relationships. First, cumulative accepted-valid exposure determines the rate
at which distinct artifacts are discovered. Second, total valid mass and
verification quality determine how much raw generation is required to
produce that exposure.

After establishing the general exposure-to-discovery law, we return to the
anchored recursive-retraining result from
Section~\ref{sec:recursive-exposure} and show how it induces the exposure
structure needed for adaptive discovery-rate and cost scaling.

\subsection{Cumulative Exposure and Discovery Scaling}

For the rate analysis, assume that valid artifacts have a uniform
per-occurrence acceptance probability \(r_{t,k}=r_t\). Recall that
\[
M_t^{\mathrm{val}}
=
Q_t(\K)
\]
is the probability that a generated candidate is valid. Define the cumulative
accepted-valid exposure through iteration \(T\) by
\[
E_T
=
\sum_{t=0}^{T-1}
Nr_tM_t^{\mathrm{val}}.
\]
This quantity counts accepted valid generation opportunities with
multiplicity, including repeated occurrences of previously discovered
artifacts.

For each \(k\in\K\), define its artifact-level cumulative accepted exposure
by
\[
\Lambda_{T,k}
=
\sum_{t=0}^{T-1}Nr_tQ_t(k).
\]
The artifact-level exposures partition the total exposure:
\[
\sum_{k\in\K}\Lambda_{T,k}
=
E_T.
\]

For simplicity, we state the result for \(\K_0^+=\varnothing\). With a
nonempty initial retained set, the same analysis applies after restricting
the target domain to \(\K\setminus\K_0^+\). For a fixed generation and
verification schedule \(\{Q_t,r_t\}_{t=0}^{T-1}\), with independent candidate
draws and verifier decisions, the probability that artifact \(k\) is accepted
at least once by iteration \(T\) is
\[
1-
\prod_{t=0}^{T-1}
\left(1-r_tQ_t(k)\right)^N.
\]
Consequently,
\[
\mathbb E[D_T]
=
\sum_{k\in\K}
\left[
1-
\prod_{t=0}^{T-1}
\left(1-r_tQ_t(k)\right)^N
\right],
\qquad
D_T=|\K_T^+|.
\]

When \(E_T>0\), define the cumulative accepted-valid exposure distribution
\[
\overline q_T(k)
=
\frac{\Lambda_{T,k}}{E_T},
\qquad
k\in\K.
\]

\begin{assumption}[Zipf Cumulative Exposure Tail]
\label{asm:tail-equiv}
For each relevant horizon \(T\), rank the valid artifacts as
\(k_{T,1},k_{T,2},\ldots\) in decreasing
\(\overline q_T\)-probability. There exist constants
\(\alpha>1\) and \(0<c_1\le c_2<\infty\), independent of \(T\) and of the
relevant tail rank, such that
\[
c_1j^{-\alpha}
\le
\overline q_T(k_{T,j})
\le
c_2j^{-\alpha}
\]
throughout the pre-saturation tail regime.
\end{assumption}
For a finite knowledge domain, the \emph{pre-saturation regime} refers to
horizons for which the characteristic occupancy scale
$E_T^{1/\alpha}$ remains small relative to \(|\K|\), so that finite-domain truncation does not
yet dominate the discovery rate. For a countably infinite domain, this
restriction is unnecessary.

\begin{remark}[Role of the ideal difficulty distribution]
The rate analysis is stated directly in terms of the cumulative exposure
distribution \(\overline q_T\) induced by the actual NOVA process. The ideal
difficulty distribution \(P\) is therefore not required by the occupancy
argument. It can still provide a domain-level interpretation when
\(\overline q_T\), or the valid restriction of the base distribution \(B\),
is comparable to \(P\) over the relevant tail.
\end{remark}

\begin{proposition}[Zipf Occupancy Scaling
{\citep[see, e.g.,][]{karlin1967,ben-hamou2017}}]
\label{prop:zipf-decay}
Let \(p_j\asymp j^{-\alpha}\) with \(\alpha>1\). Then, in the
pre-saturation regime,
\[
\sum_j\left(1-e^{-Sp_j}\right)
=
\Theta\!\left(S^{1/\alpha}\right)
\]
and
\[
\sum_j\min\{1,Sp_j\}
=
\Theta\!\left(S^{1/\alpha}\right).
\]
\end{proposition}

\begin{theorem}[Zipf Discovery Rate and Cost Scaling]
\label{thm:rate-bound}
Consider a fixed generation and verification schedule
\(\{Q_t,r_t\}_{t=0}^{T-1}\), with independent candidate draws and verifier
decisions. Suppose that the cumulative accepted-valid exposure distribution
\(\overline q_T\) satisfies Assumption~\ref{asm:tail-equiv} with exponent
\(\alpha>1\). Then, in the pre-saturation regime,
\[
\mathbb E[D_T]
=
\Theta\!\left(E_T^{1/\alpha}\right).
\]

If, in addition, there exist constants \(m_{\min}>0\) and \(r_{\min}>0\)
such that
\[
M_t^{\mathrm{val}}\ge m_{\min},
\qquad
r_t\ge r_{\min}
\]
over the discovery horizon, then
\[
E_T
=
\Theta(NT).
\]
Equivalently, at the expectation level, the raw generation cost corresponding
to \(D\) distinct genuine discoveries satisfies
\[
R_{\rm cum}(D)
=
\Theta(c_{\rm gen}D^\alpha).
\]
\end{theorem}

\begin{proof}
For every \(k\in\K\),
\[
1-
\prod_{t=0}^{T-1}
\left(1-r_tQ_t(k)\right)^N
\ge
1-e^{-\Lambda_{T,k}},
\]
because \(1-x\le e^{-x}\). A union bound also gives
\[
1-
\prod_{t=0}^{T-1}
\left(1-r_tQ_t(k)\right)^N
\le
\min\{1,\Lambda_{T,k}\}.
\]
Therefore,
\[
\sum_{k\in\K}
\left(1-e^{-\Lambda_{T,k}}\right)
\le
\mathbb E[D_T]
\le
\sum_{k\in\K}
\min\{1,\Lambda_{T,k}\}.
\]
Since
\[
\Lambda_{T,k}
=
E_T\overline q_T(k)
\]
and
\[
\overline q_T(k_{T,j})
\asymp
j^{-\alpha},
\]
Proposition~\ref{prop:zipf-decay} shows that both bounds are
\(\Theta(E_T^{1/\alpha})\). Under the nonvanishing valid-mass and acceptance assumptions,
\[
Nr_{\min}m_{\min}T
\le
E_T
=
\sum_{t=0}^{T-1}
Nr_tM_t^{\mathrm{val}}
\le
NT,
\]
and hence $E_T=\Theta(NT)$.
The raw generation cost through iteration \(T\) is \(c_{\rm gen}NT\), and is
therefore \(\Theta(c_{\rm gen}E_T)\). Inverting
\(D=\Theta(E_T^{1/\alpha})\) at the expectation level gives
\(E_T=\Theta(D^\alpha)\), proving
\[
R_{\rm cum}(D)
=
\Theta(c_{\rm gen}D^\alpha).
\]
\end{proof}

\begin{remark}[Role and scope of the general discovery-scaling theorem]
\label{rem:rate-theorem-role}
Theorem~\ref{thm:rate-bound} is an exposure-to-discovery theorem, not a
retraining theorem. It takes the cumulative accepted-valid exposure profile
generated by the NOVA process as its input and determines how efficiently
that exposure is converted into distinct discoveries.

Its first conclusion,
\[
\mathbb E[D_T]
=
\Theta\!\left(E_T^{1/\alpha}\right),
\]
describes an occupancy bottleneck. Although cumulative accepted-valid
exposure grows as \(E_T\), repeated exposure to artifacts already encountered
means that only order \(E_T^{1/\alpha}\) distinct artifacts are discovered.
The exponent \(\alpha\) therefore determines how efficiently useful exposure
is converted into distinct discoveries.

Its second conclusion,
\[
R_{\rm cum}(D)
=
\Theta(c_{\rm gen}D^\alpha),
\]
requires the additional conditions
\[
M_t^{\mathrm{val}}\ge m_{\min}>0
\qquad\text{and}\qquad
r_t\ge r_{\min}>0.
\]
These assumptions ensure that a constant fraction of raw generations becomes
accepted-valid exposure, so that
\[
E_T=\Theta(NT).
\]
If either the valid mass or the acceptance probability vanishes, the
exposure-level law may remain informative while the raw generation-cost law
need not hold.

Theorem~\ref{thm:rate-bound} does not assert that an arbitrary retraining
mechanism produces a Zipf cumulative exposure tail, nor does it require
\(Q_t\) to converge to the ideal distribution \(P\). It characterizes the
consequences of a specified cumulative exposure structure. The adaptive
anchored-retraining result in
Corollary~\ref{cor:anchored-scaling} will show how that structure can arise
from an explicit retraining mechanism.
\end{remark}

\begin{remark}[Tail concentration and diminishing returns]
The exponent \(\alpha\) characterizes the concentration of cumulative
accepted-valid exposure, rather than intrinsic domain difficulty by itself.
A larger \(\alpha\) concentrates exposure on relatively few high-probability
artifacts. Those artifacts can be discovered quickly, but repeated generation
increasingly returns artifacts that have already been found, while the
remaining distinct opportunities lie deeper in a thinner tail.

Consequently, the average generation cost per discovery satisfies
\[
\frac{R_{\rm cum}(D)}{D}
=
\Theta(c_{\rm gen}D^{\alpha-1}),
\]
which increases with \(D\) whenever \(\alpha>1\). Under a smooth power-law
approximation, the marginal cost has the same scaling:
\[
\frac{dR_{\rm cum}}{dD}
=
\Theta(c_{\rm gen}D^{\alpha-1}).
\]

As \(\alpha\downarrow1\), accepted-valid exposure is distributed across a
heavier tail. Distinct opportunities persist deeper into the discovery
process, and both average and marginal discovery costs grow more slowly with
\(D\). Thus a process can make its highest-exposure discoveries quickly while
still facing severe diminishing returns at the later discovery frontier.
\end{remark}

\begin{remark}[Fixed versus adaptive generator sequences]
\label{rem:fixed-versus-adaptive}
The exact occupancy identity above is stated for a fixed generation and
verification schedule. For an arbitrary history-dependent generator
sequence, it cannot generally be recovered by conditioning on the realized
future sequence \(\{Q_t\}\), because \(Q_{t+1}\) may itself depend on whether
an artifact was discovered at an earlier iteration.

Theorem~\ref{thm:rate-bound} therefore gives the scaling law for a fixed
exposure schedule. Corollary~\ref{cor:anchored-scaling} below gives an
adaptive counterpart for the canonical anchored-retraining model, using
direct conditional probability bounds rather than conditioning on the
realized future generator sequence.
\end{remark}

\subsection{Adaptive Discovery Scaling under Anchored Retraining}
\label{sec:adaptive-scaling}

Theorem~\ref{thm:feedback} was placed in
Section~\ref{sec:recursive-exposure} because its first role is to determine
whether recursive retraining satisfies the persistent-exposure condition
needed for coverage. Its anchored conclusion has a second, quantitative role:
it keeps every \(Q_t\) uniformly comparable to the persistent base
distribution \(B\).

We now show that this comparison transfers the valid-base tail to cumulative
accepted-valid exposure and yields an adaptive counterpart of
Theorem~\ref{thm:rate-bound}.

\begin{corollary}[Adaptive Discovery Rate and Cost under Anchored Retraining]
\label{cor:anchored-scaling}
Assume the anchored setting of Theorem~\ref{thm:feedback} and monotone
accumulation. Conditional on \(\mathcal F_t\), assume that the \(N\)
candidates at iteration \(t\) are sampled independently from \(Q_t\), and
that verifier decisions are conditionally independent given the generated
candidates.

Suppose that valid artifacts have a uniform, \(\mathcal F_t\)-measurable
per-occurrence acceptance probability \(r_{t,k}=r_t\), with
\[
0<r_{\min}
\le
r_t
\le
r_{\max}
\le
1.
\]
Assume also that $B(\K)>0$ and, for simplicity, $\K_0^+=\varnothing$.\footnote{With a nonempty initial retained set, the same conclusions apply after
restricting the target domain to \(\K\setminus\K_0^+\).} Define the valid restriction of the base distribution by
\[
B_{\K}(k)
=
\frac{B(k)}{B(\K)},
\qquad
k\in\K.
\]
Then, for every horizon \(T\), on every realized history,
\[
e^{-2C}B_{\K}(k)
\le
\overline q_T(k)
\le
e^{2C}B_{\K}(k).
\]

If
\[
B_{\K}(k_j)
\asymp
j^{-\alpha},
\qquad
\alpha>1,
\]
then, in the pre-saturation regime,
\[
\mathbb E[D_T]
=
\Theta\!\left((NT)^{1/\alpha}\right).
\]
Equivalently, at the expectation level, the raw generation cost corresponding
to \(D\) distinct genuine discoveries satisfies
\[
R_{\rm cum}(D)
=
\Theta(c_{\rm gen}D^\alpha).
\]
\end{corollary}

The proof is given in Appendix~\ref{app:feedback-proofs}.

\begin{remark}[Role of the adaptive anchored-retraining corollary]
\label{rem:anchored-scaling-role}
Theorem~\ref{thm:rate-bound} and
Corollary~\ref{cor:anchored-scaling} play complementary roles.

Theorem~\ref{thm:rate-bound} is mechanism-agnostic. It identifies the general
occupancy law that converts a fixed cumulative accepted-valid exposure
profile into distinct discoveries. Its Zipf-tail condition is imposed
directly on the cumulative exposure distribution \(\overline q_T\).

Corollary~\ref{cor:anchored-scaling} is mechanism-specific and adaptive. It
shows that, under the dynamics of Theorem~\ref{thm:feedback}, the required
cumulative exposure tail need not be imposed independently. The pointwise
comparison
\[
e^{-C}B(k)
\le
Q_t(k)
\le
e^CB(k)
\]
implies
\[
e^{-2C}B_{\K}(k)
\le
\overline q_T(k)
\le
e^{2C}B_{\K}(k).
\]
Thus the tail structure of the persistent valid base distribution is
transferred to the cumulative accepted-valid exposure distribution.

The adaptive rate conclusion is proved directly through conditional
discovery-probability bounds. This avoids the invalid step of conditioning
on the realized future sequence \(\{Q_t\}\), which may itself depend on
earlier discovery events.

Finally, the cost expression
\[
R_{\rm cum}(D)
=
\Theta(c_{\rm gen}D^\alpha)
\]
is an expectation-level inversion of
\[
\mathbb E[D_T]
=
\Theta\!\left((NT)^{1/\alpha}\right).
\]
It is not a claim that the expected hitting time of the \(D\)-th discovery is
necessarily \(\Theta(D^\alpha)\).
\end{remark}

\begin{remark}[Relation to the ideal difficulty distribution]
\label{rem:ideal-difficulty}
If the valid base distribution \(B_{\K}\) is comparable to the ideal
difficulty distribution \(P\) over the relevant tail, then
Corollary~\ref{cor:anchored-scaling} transfers the Zipf exponent of \(P\) to
the cumulative accepted-valid exposure distribution.

Such comparability gives the exponent \(\alpha\) a domain-level difficulty
interpretation, but it is not required by
Theorem~\ref{thm:rate-bound},
Theorem~\ref{thm:feedback}, or
Corollary~\ref{cor:anchored-scaling}.
\end{remark}

\begin{remark}[Compute sustainability]
\label{rem:compute-sustainability}
The discovery-cost law in Theorem~\ref{thm:rate-bound}, and its adaptive
counterpart in Corollary~\ref{cor:anchored-scaling}, give a simple
sustainability test. Sustaining a target discovery trajectory \(D(t)\)
requires a cumulative generation budget of order
\[
\mathcal C(t)
\gtrsim
c_{\rm gen}D(t)^\alpha.
\]
Thus, when \(\alpha>1\), aggressive discovery trajectories require
superlinear growth in available compute. For a transformer with
\(n_{\rm par}\) parameters and expected output length \(L\), one may use the
rough inference-cost approximation
\[
c_{\rm gen}
\approx
2n_{\rm par}L
\]
FLOPs per candidate. This changes the multiplicative cost scale but not the
discovery exponent \(\alpha\). Appendix~\ref{app:compute-sustainability}
provides illustrative compute-growth scenarios.
\end{remark}


\section{Human--AI Collaborative NOVA}
\label{sec:human}

The preceding sections identify four barriers for autonomous NOVA: reachable
support limits what can be discovered
(Corollary~\ref{cor:barrier}); unanchored recursive retraining can reinforce
early accepted outputs and suppress undiscovered alternatives
(Theorem~\ref{thm:feedback}); imperfect verification becomes more dangerous
as new-valid mass shrinks (Corollary~\ref{cor:trap}); and Zipf occupancy makes
cumulative discovery costs grow superlinearly
(Theorem~\ref{thm:rate-bound}). Human experts can intervene at these
bottlenecks by expanding support, redirecting probability mass, counteracting
recursive narrowing, proposing novel candidates, and providing high-precision
verification when formal verifiers are unavailable.

This formalizes the qualitative observation emphasized by \citet{tao2026}: in frontier mathematical practice, AI systems are most effective when embedded in expert-guided workflows rather than treated as fully autonomous discovery engines.

In the augmented NOVA loop, detailed in Appendix~\ref{app:human-details}, a human expert modifies $\Qt$ to a guided distribution $Q_t'$, the AI generates $N_{\mathrm{AI}}$ candidates from $Q_t'$, the expert generates $N_H$ candidates from their own distribution $P_H$, and the combined candidates are verified. Let
\[
M_t^{\mathrm{new,guided}}
=
Q_t'(\K\setminus \K_t^+),
\qquad
M_t^{\mathrm{new},H}
=
P_H(\K\setminus \K_t^+)
\]
denote the new-valid mass under the guided AI distribution and the human proposal distribution, respectively. Let \(r_{\mathrm{eff},t}\) denote the effective true-positive rate for new valid AI-generated candidates after human review:
\[
r_{\mathrm{eff},t}
=
\Pr[V_{\mathrm{eff}}(C_t)=1
\mid C_t\in \K\setminus \K_t^+,\, C_t\sim Q_t'].
\]
Define the \emph{human amplification factor}\label{def:amplification} as
\[
A_H =
\frac{\mathbb{E}[|S_t^{H+\mathrm{AI}}|]}
{\mathbb{E}[|S_t^{\mathrm{AI}}|]}.
\]

\begin{theorem}[Sparse-Regime Human Amplification]
\label{thm:amplification}
In the sparse regime, neglecting duplicate discoveries between human- and AI-generated candidates, the amplification factor admits the first-order decomposition
\[
A_H = A_{\mathrm{guide}}\cdot A_{\mathrm{verify}}\cdot A_{\mathrm{gen}},
\]
where
\[
A_{\mathrm{guide}}
=
\frac{M_t^{\mathrm{new,guided}}}{M_t^{\mathrm{new}}},
\qquad
A_{\mathrm{verify}}
=
\frac{r_{\mathrm{eff},t}}{r_t},
\qquad
A_{\mathrm{gen}}
=
1+
\frac{N_H \rho_{H,t} M_t^{\mathrm{new},H}}
{N_{\mathrm{AI}} r_{\mathrm{eff},t} M_t^{\mathrm{new,guided}}}.
\]
Here \(r_t\) is the autonomous true-positive rate, \(r_{\mathrm{eff},t}\) is the effective true-positive rate for guided AI-generated candidates after human review, and \(\rho_{H,t}\) is  acceptance rate for valid human-generated candidates.
\end{theorem}

Theorem~\ref{thm:amplification} separates three human contributions: guidance raises the new-valid mass reached by AI generation, verification improves acceptance of valid AI artifacts, and generation adds human-proposed candidates. Together, these factors suggest a copilot regime in which AI performs high-volume search within guided regions while the human provides frontier guidance, hypotheses, and difficult verification. Guidance is especially central because it
can change not only discovery probability within the current support, but also the support itself. Next theorem formalizes this effect.

\begin{theorem}[Human-Guided Support Expansion]
\label{thm:human-support}
If human guidance changes \(Q_t\) to \(Q_t'\) such that
\(K \cap \supp(Q_t) \subsetneq K \cap \supp(Q_t')\), then the reachable valid set
strictly expands. Hence guidance can break the autonomous exploration barrier
whenever \(Q_t'\) assigns non-negligible mass to valid artifacts outside the
previous effective support while preserving the previously reachable valid support.
\end{theorem}

Appendix~\ref{app:human-details} gives the details of augmented NOVA loop, proof of Theorems~\ref{thm:amplification} and \ref{thm:human-support}, and a simple effort-allocation principle for deciding when human effort should be spent.

\section{Conclusions, Limitations, Broader Impacts, and Future Directions}
\label{sec:conclusion}

We formalized the generate--verify--accumulate--retrain loop as an adaptive
sampling process whose generator evolves through accepted outputs. NOVA first
identifies general conditions on this adaptive sequence for coverage,
contamination control, retention, and discovery-cost scaling. It then shows
how those conditions can arise or fail under an explicit recursive retraining
model. Unanchored feedback can lock generation onto early accepted artifacts
and suppress initially reachable alternatives. Anchoring to a persistent base
distribution prevents unbounded distortion, guarantees continued exposure,
and preserves the tail structure required for occupancy-based discovery
scaling. Reliable AI-driven discovery therefore requires exploration,
verification, accumulation, and retraining to remain jointly aligned. Human
guidance can further redirect search, expand reachable support, counteract
recursive narrowing, and provide reliable verification near the autonomous
frontier.

\textbf{Limitations.} NOVA is a stylized framework, and its results should be interpreted accordingly.
First, the coverage theorem gives sufficient, not necessary, conditions for
finite-domain discovery; it isolates failure modes but does not imply rapid
discovery or practical feasibility. Second, the cost law
$
R_{\rm cum}(D)=\Theta(c_{\rm gen}D^\alpha)
$
is not universal. It depends on a Zipf-like cumulative distribution of
repeat-inclusive accepted valid exposure. The anchored feedback result derives
this condition when the valid portion of the persistent base distribution has
such a tail, but other adaptive generator sequences or retraining maps may
produce a different exposure profile and require a different occupancy law. The KL-regularized recursive retraining model is a distribution-level
abstraction of mechanisms such as recursive fine-tuning, replay, retrieval
conditioning, or reward-based reweighting. It isolates how reinforcement from
accepted artifacts competes with persistent exploratory support, but it does
not model the full neural-network optimization dynamics that produce
\(Q_{t+1}\). Third, verification
is modeled abstractly through true- and false-positive behavior. This captures
local contamination effects, but not the full complexity of proof checking,
simulation, experimentation, human judgment, correlated errors, or verification
latency. Finally, NOVA treats discoveries as discrete artifacts sampled from a
knowledge space. This abstraction does not capture all aspects of scientific
progress, such as conceptual reframing, new measurements, causal experimentation, etc. Thus NOVA should be viewed as a foundation for
analyzing discovery loops, rather than a complete theory of discovery.

\textbf{Broader impacts.}
On the positive side, a theory of
generate--verify--accumulate--retrain loops may help design more reliable
AI-assisted discovery systems for mathematics, science, and engineering by making
failure modes, verification requirements, and human guidance explicit. On the
negative side, the same loops could be misused in harmful domains, or could lead
to over-reliance on autonomous systems when verification is weak. In particular,
false positives and contaminated retained artifacts may be recursively
amplified through future generation. Our analysis highlights verification, contamination control,
support limits, and human oversight as important safeguards.

\textbf{Open problems.} There are several immediate future directions.
\textbf{(1) Collaborative discovery:} How does effective missing mass scale with \(m\) models with diverse supports, and can composition yield superlinear gains? \textbf{(2) Verification difficulty:} Can we formalize how the difficulty of checking candidates (from mechanical proof checking to code testing, noisy experiments, and subjective evaluation) limits the knowledge level achievable by autonomous NOVA? \textbf{(3) Information limits:}  For the subset \(\K^*\subseteq\K\) discoverable from the initial data, model, and allowed NOVA operations, can \(|\K^*|\) be bounded by the information that \(\mathcal D_0\) contains about \(\K\), e.g., \(|\K^*|\lesssim 2^{I(\mathcal D_0;\K)}\)?

\bibliographystyle{plainnat}  
\bibliography{references}

\newpage
\appendix

\section{Motivating Examples}
\label{app:examples}

\paragraph{Example 1: Discovering new mathematical proofs.}
Consider $\K$ as the set of valid formal proofs of mathematical conjectures. A language model $\mathcal{M}_t$ generates candidate proofs in a formal language (e.g., Lean~4). The verification step is performed by a proof assistant: the Lean type-checker mechanically checks whether each candidate is a well-typed proof of the stated theorem. In this setting, verification is effectively perfect relative to the formal specification: false positives are mechanically ruled out, so $\delta_t=0$, while artifact-wise acceptance satisfies $r_{t,k}=1$ for any correctly generated formal proof $k$. This is the ideal verification regime for NOVA. Recent systems such as AlphaProof~\citep{alphaproof2025} and DeepSeek-Prover-V2~\citep{deepseekv2_2025} operate in this regime.

\paragraph{Example 2: Discovering new molecules and materials.}
Consider $\K$ as the set of molecules with a desired functional property. A generative model proposes candidate molecular structures, and verification is performed through wet-lab experiments or computational simulations. Here, verification is stochastic: artifact-wise true-positive rates $r_{t,k}$ may be less than one and false-positive rates $\delta_t$ may be nonzero. This places the system in the imperfect verification regime analyzed in Section~\ref{sec:convergence}, where the contamination dynamics of Proposition~\ref{thm:contamination} become relevant. Recent AI-driven discovery platforms for proteins~\citep{jumper2021alphafold} and materials~\citep{merchant2023gnome} exemplify this setting.

\paragraph{Example 3: Discovering new scientific hypotheses.}
At the frontier of knowledge, consider $\K$ as the set of true scientific hypotheses. Verification may be infeasible: testing a hypothesis may require experiments that take years or new instrumentation. In the extreme case, $r_{t,k}\approx 0$ for the most novel artifacts, making autonomous NOVA essentially impossible. It is precisely in this regime that human expert guidance (Section~\ref{sec:human}) becomes indispensable.

\section{Missing Mass Estimation Details}
\label{app:missing-mass}

This appendix clarifies the different missing-mass quantities that appear in NOVA. The
classical Good--Turing estimator applies to the batch unseen mass under the current
generator $Q_t$: the probability that another sample from $Q_t$ would produce an artifact
not seen in the current batch. This is distinct from the historical new-valid mass
$M_t^{\rm new}$, which is the probability that the generator produces a valid artifact not
yet accumulated in $K_t^+$. We formalize this distinction, give the exact one-step
discovery rate and its sparse-regime approximation, and recall the Good--Toulmin
forecasting formula for fixed-generator batch extrapolation.

\subsection{Good--Turing Estimator}

\begin{theorem}[{Good--Turing Estimator for Batch Missing Mass -- \cite{good1953} and \cite{mcallester2000}}]
\label{thm:good-turing}
Let $X_1,\ldots,X_N$ be i.i.d.\ from a discrete distribution $Q$, and let
\[
M_N=\sum_x Q(x)\mathbf 1[x\notin\{X_1,\ldots,X_N\}]
\]
be the missing mass after the batch. The Good--Turing estimator $f_1/N$, where $f_1$ is the number of species observed exactly once, is the classical estimator for $M_N$, the probability that the next draw belongs to a species unseen in the current batch.
\end{theorem}

\begin{remark}[Concentration]
Standard missing-mass and Good--Turing concentration results give additive error bounds of order $N^{-1/2}$, up to logarithmic factors and estimator-specific bias terms; see the cited references for precise statements.
\end{remark}

\begin{proposition}[Good--Turing Estimates Ambient Batch Unseen Mass]
\label{thm:nova-batch-mass}
In the NOVA framework, at iteration $t$ the model generates $N$ candidates i.i.d.\ from $\Qt$ over $\X$. Since $\Qt$ may assign positive mass to invalid candidates ($U_t > 0$), the Good--Turing estimator $f_1^{(t)}/N$ estimates the ambient batch unseen mass $M_{t,\X}^{\mathrm{batch}}$, providing an upper bound on the valid-artifact component: $M_{t,K}^{\mathrm{batch}} \leq M_{t,\X}^{\mathrm{batch}}$.
\end{proposition}

\subsection{Estimating New-Valid Mass}

If $\Kp$ is known explicitly, a practical unbiased estimator from one batch is:
\begin{equation}
\widehat{M}_t^{\mathrm{new,MC}} = \frac{1}{N} \sum_{i=1}^{N} \mathbf{1}[X_i \in \K \setminus \Kp]
\end{equation}
This is unbiased for $M_t^{\mathrm{new}}$ when $X_i \sim \Qt$, since $\mathbb{E}[\mathbf{1}[X_i \in \K \setminus \Kp]] = M_t^{\mathrm{new}}$. The indicator requires checking both that $X_i$ is valid ($X_i \in \K$) and not already discovered ($X_i \notin \Kp$).

\subsection{Good--Toulmin Forecasting}

\begin{theorem}[Good--Toulmin Forecasting (Batch Extrapolation) -- \cite{good1956}]
\label{thm:good-toulmin}
At fixed iteration $t$, suppose $N$ candidates are sampled i.i.d. from a fixed
generator $Q_t$, and let $f_r^{(t)}$ be the number of artifacts observed exactly
$r$ times. For an additional sample of size $sN$, the Good--Toulmin formal series
estimates the expected number of new species by
\begin{equation}
\sum_{r\ge 1}(-s)^{r+1}f_r^{(t)} .
\end{equation}

\end{theorem}

Note that this is a fixed-generator batch extrapolation formula. Once $Q_t$ changes across
iterations due to retraining, it no longer applies directly without additional
stability assumptions.

\section{Convergence Proof}
\label{app:convergence-proofs}

\begin{proof}[Proof of Theorem~\ref{thm:convergence}]
\textbf{Monotone growth and absence of contamination:}
By C1,
\[
\K_t^+\subseteq\K_{t+1}^+
\qquad\text{for all }t,
\]
so already discovered genuine artifacts are never lost. By C4, no invalid
artifact is accepted. Hence the retained state is uncontaminated and
\[
\widehat{\K}_t=\K_t^+\subseteq\K
\qquad\text{for all }t.
\]

\textbf{Eventual discovery of a fixed artifact:}
Fix \(k\in\K\), and let
\[
A_{t,k}
=
\left\{
k\notin\K_t^+
\text{ and }
k\in\K_{t+1}^+
\right\}
\]
denote the event that \(k\) is newly discovered at iteration \(t\).
Recall that \(\mathcal F_t\) is the history before generation at iteration
\(t\).

Conditional on \(\mathcal F_t\), if \(k\notin\K_t^+\), consider the event
that the first candidate in the batch equals \(k\) and is accepted. By the
definition of \(r_{t,k}\),
\[
\Pr\!\left[
c_{t,1}=k
\text{ and }
V(c_{t,1})=1
\mid
\mathcal F_t
\right]
=
Q_t(k)r_{t,k}.
\]
Since this event is sufficient for \(k\) to be newly discovered, C3 gives
\[
\Pr\!\left(A_{t,k}\mid\mathcal F_t\right)
\ge
r_{\min}Q_t(k)
\mathbf 1\{k\notin\K_t^+\}.
\]

Let
\[
\tau_k
=
\inf\{t:k\in\K_t^+\}
\]
be the discovery time of \(k\). On the event \(\{\tau_k=\infty\}\), artifact
\(k\) remains undiscovered at every iteration. Hence C2 implies
\[
\sum_{t=0}^{\infty}
\left[
1-\left(1-Q_t(k)\right)^N
\right]
=
\infty.
\]
For every \(q\in[0,1]\),
\[
1-(1-q)^N
\le
Nq.
\]
Therefore, on every sample path for which \(\tau_k=\infty\),
\[
\sum_{t=0}^{\infty}Q_t(k)
=
\infty.
\]
It follows that
\[
\sum_{t=0}^{\infty}
\Pr\!\left(A_{t,k}\mid\mathcal F_t\right)
\ge
r_{\min}
\sum_{t=0}^{\infty}Q_t(k)
=
\infty.
\]

By the conditional Borel--Cantelli lemma, on the event where this conditional-probability sum
diverges, the events \(A_{t,k}\) occur infinitely often almost surely. But
\(A_{t,k}\) can occur at most once, because C1 ensures that once \(k\) is
discovered it remains in \(\K_t^+\). Therefore,
\[
\Pr(\tau_k=\infty)=0.
\]
Thus every fixed \(k\in\K\) is eventually discovered almost surely.

\textbf{Convergence for finite \(\K\):}
Since \(|\K|<\infty\), the finite intersection of the probability-one events
\[
\{\tau_k<\infty\},
\qquad k\in\K,
\]
also has probability one. Therefore every valid artifact is eventually
discovered almost surely. By monotone accumulation,
\[
\K_t^+\to\K
\qquad\text{almost surely}.
\]

\textbf{Countably infinite extension:}
If \(\K\) is countably infinite and C2--C3 hold for every \(k\in\K\), the
same artifact-wise argument gives
\[
\Pr(\tau_k<\infty)=1
\qquad\text{for every }k\in\K.
\]
Taking a countable intersection shows that, almost surely, every artifact is
eventually discovered. This is pointwise asymptotic coverage: there need not
exist any finite iteration at which all artifacts have already been
discovered.
\end{proof}

\begin{remark}
Note that C4 requires $\delta_t = 0$ (no false positives), \emph{not} perfect recall. False negatives can slow discovery, but they do not corrupt the retained knowledge base; false positives do. This distinction is critical for the contamination analysis in Section~\ref{sec:convergence}.
\end{remark}

\begin{remark}[Exploration Barrier]
The exploration barrier (Corollary~\ref{cor:barrier}) is fundamental but not insurmountable: human guidance can expand reachable support, as formalized in Section~\ref{sec:human}.
\end{remark}

\section{Recursive Feedback Proofs}
\label{app:feedback-proofs}

\begin{proof}[Proof of Theorem~\ref{thm:feedback}]
We prove the unanchored and anchored claims separately.

\paragraph{Unanchored recursive feedback.}
When \(\lambda=0\), the update becomes
\[
Q_{t+1}(x)
=
\frac{Q_t(x)e^{\eta s_t(x)}}{Z_t}.
\]
Hence, while \(k\) remains undiscovered,
\[
\frac{Q_{t+1}(k)}{Q_{t+1}(i)}
=
\frac{Q_t(k)}{Q_t(i)}
e^{-\eta(s_t(i)-s_t(k))}
\le
e^{-\eta\Delta}
\frac{Q_t(k)}{Q_t(i)}.
\]
Iterating from \(t_0\) gives
\[
\frac{Q_t(k)}{Q_t(i)}
\le
\frac{Q_{t_0}(k)}{Q_{t_0}(i)}
e^{-\eta\Delta(t-t_0)}.
\]
Since \(Q_t(i)\le1\),
\[
Q_t(k)
\le
\frac{Q_{t_0}(k)}{Q_{t_0}(i)}
e^{-\eta\Delta(t-t_0)},
\]
and therefore
\[
\sum_{t=t_0}^{\infty}NQ_t(k)
<
\infty
\]
on every sample path along which \(k\) remains undiscovered.

It remains to show that such a sample path has positive probability. Define
\[
R_t
=
\frac{Q_{t_0}(k)}{Q_{t_0}(i)}
e^{-\eta\Delta(t-t_0)}.
\]
While \(k\) remains undiscovered, the preceding ratio bound gives
\[
\frac{Q_t(k)}{Q_t(i)}
\le
R_t.
\]
Since
\[
Q_t(i)+Q_t(k)\le1,
\]
we obtain
\[
Q_t(k)
\le
R_tQ_t(i)
\le
R_t\bigl(1-Q_t(k)\bigr),
\]
and hence
\[
Q_t(k)
\le
\frac{R_t}{1+R_t}.
\]
Therefore, conditional on \(k\) remaining undiscovered before iteration \(t\),
the probability that \(k\) is not generated in any of the \(N\) candidate
slots is at least
\[
\left(1-Q_t(k)\right)^N
\ge
\left(1+R_t\right)^{-N}.
\]
Sequential conditioning then gives the following lower bound on the
probability that \(k\) is never generated after iteration \(t_0\):
\[
\prod_{t=t_0}^{\infty}
\left(1+R_t\right)^{-N}.
\]
Because
\[
\sum_{t=t_0}^{\infty}R_t
<
\infty
\]
and
\[
\log(1+R_t)\le R_t,
\]
we have
\[
\sum_{t=t_0}^{\infty}\log(1+R_t)
<
\infty.
\]
Consequently,
\[
\prod_{t=t_0}^{\infty}
\left(1+R_t\right)^{-N}
>
0.
\]
Thus \(k\) is never generated, and therefore remains undiscovered, with
positive probability.

\paragraph{Anchored recursive feedback.}
Now suppose \(\lambda>0\). Let
\[
\mathcal S
=
\operatorname{supp}(Q_0)
=
\operatorname{supp}(B)
\]
denote the common support assumed in Theorem~\ref{thm:feedback}. Both
\(Q_0\) and \(B\) are strictly positive on \(\mathcal S\), and the recursive
update preserves this support. We interpret the update on \(\mathcal S\),
with \(Q_t(x)=B(x)=0\) outside \(\mathcal S\). Therefore, for every
\(x\in\mathcal S\), the logarithmic ratio
\[
h_t(x)
=
\log\frac{Q_t(x)}{B(x)}
\]
is well-defined.

Using
\[
Q_{t+1}(x)
=
\frac{
Q_t(x)^\rho
B(x)^{1-\rho}
e^{\gamma s_t(x)}
}{
Z_t
},
\]
we obtain, for every \(x,y\in\mathcal S\),
\[
h_{t+1}(x)-h_{t+1}(y)
=
\rho\bigl(h_t(x)-h_t(y)\bigr)
+
\gamma\bigl(s_t(x)-s_t(y)\bigr).
\]
Define
\[
C_t
=
\sup_{x\in\mathcal S}h_t(x)
-
\inf_{x\in\mathcal S}h_t(x).
\]
By the bounded-score-oscillation assumption of
Theorem~\ref{thm:feedback},
\[
\sup_{x\in\mathcal S}s_t(x)
-
\inf_{x\in\mathcal S}s_t(x)
\le
L.
\]
Therefore,
\[
C_{t+1}
\le
\rho C_t+\gamma L.
\]
Because
\[
\rho
=
\frac{1}{1+\lambda},
\qquad
\gamma
=
\frac{\eta}{1+\lambda},
\]
we have
\[
\frac{\gamma L}{1-\rho}
=
\frac{\eta L}{\lambda}.
\]
Iterating the recursion gives
\[
C_t
\le
\rho^tC_0
+
\gamma L\sum_{\ell=0}^{t-1}\rho^\ell
=
\rho^tC_0
+
\left(1-\rho^t\right)\frac{\eta L}{\lambda}.
\]
Therefore,
\[
C_t
\le
C,
\qquad
C
=
\max\left\{
C_0,\frac{\eta L}{\lambda}
\right\}.
\]

Finally, because \(Q_t\) and \(B\) are supported on \(\mathcal S\),
\[
\sum_{x\in\mathcal S}B(x)e^{h_t(x)}
=
\sum_{x\in\mathcal S}Q_t(x)
=
1.
\]
It follows that
\[
\inf_{x\in\mathcal S}h_t(x)
\le
0
\le
\sup_{x\in\mathcal S}h_t(x);
\]
otherwise the \(B\)-weighted average of \(e^{h_t(x)}\) would be strictly
larger or strictly smaller than one. Since
\[
C_t
=
\sup_{x\in\mathcal S}h_t(x)
-
\inf_{x\in\mathcal S}h_t(x)
\le
C,
\]
we conclude that
\[
-C
\le
h_t(x)
\le
C
\]
for every \(x\in\mathcal S\). Exponentiating gives
\[
e^{-C}B(x)
\le
Q_t(x)
\le
e^CB(x)
\]
for every \(x\in\mathcal S\), which proves the anchored claim.
\end{proof}

\begin{proof}[Proof of Corollary~\ref{cor:anchored-coverage}]
By Theorem~\ref{thm:feedback},
\[
Q_t(k)
\ge
e^{-C}B(k)
>
0
\]
for every \(k\in\K\). Consequently,
\[
1-\left(1-Q_t(k)\right)^N
\ge
1-\left(1-e^{-C}B(k)\right)^N
>
0.
\]
Thus, for every valid artifact \(k\),
\[
\sum_{t=0}^{\infty}
\left[
1-\left(1-Q_t(k)\right)^N
\right]
=
\infty.
\]
The persistent pre-discovery exposure condition of
Theorem~\ref{thm:convergence} therefore holds. Together with nondegenerate
acceptance, monotone accumulation, and the absence of false positives,
Theorem~\ref{thm:convergence} implies that, when \(|\K|<\infty\),
\[
\K_t^+
\to
\K
\qquad\text{almost surely}.
\]
\end{proof}

\begin{proof}[Proof of Corollary~\ref{cor:anchored-scaling}]
We first compare the cumulative accepted-valid exposure distribution with the
valid restriction of the persistent base distribution, and then establish the
adaptive discovery-rate law directly.

\paragraph{Comparison of cumulative exposure distributions.}
Suppose now that \(r_{t,k}=r_t\) for valid artifacts and define
\[
A_T
=
\sum_{t=0}^{T-1}Nr_t.
\]
The pointwise comparison from Theorem~\ref{thm:feedback} gives
\[
e^{-C}B(k)
\le
Q_t(k)
\le
e^CB(k).
\]
Multiplying by \(Nr_t\) and summing over \(t\) yields
\[
e^{-C}A_TB(k)
\le
\Lambda_{T,k}
\le
e^CA_TB(k).
\]

Moreover,
\[
M_t^{\mathrm{val}}
=
\sum_{k\in\K}Q_t(k).
\]
Summing the pointwise comparison over \(k\in\K\) therefore gives
\[
e^{-C}B(\K)
\le
M_t^{\mathrm{val}}
\le
e^CB(\K).
\]
It follows that
\[
e^{-C}A_TB(\K)
\le
E_T
=
\sum_{t=0}^{T-1}Nr_tM_t^{\mathrm{val}}
\le
e^CA_TB(\K).
\]
Since
\[
\overline q_T(k)
=
\frac{\Lambda_{T,k}}{E_T}
\]
and
\[
B_{\K}(k)
=
\frac{B(k)}{B(\K)},
\]
dividing the preceding bounds gives
\[
e^{-2C}B_{\K}(k)
\le
\overline q_T(k)
\le
e^{2C}B_{\K}(k).
\]

\paragraph{Adaptive discovery-rate scaling.}
We now establish the discovery rate directly for the adaptive process,
without conditioning on the realized future sequence of generator
distributions. As in Theorem~\ref{thm:rate-bound}, assume for simplicity that
\(\K_0^+=\varnothing\); otherwise, restrict the argument to
\(\K\setminus\K_0^+\).

Fix \(k\in\K\), and define
\[
b_k
=
e^{-C}B(k).
\]
While \(k\) remains undiscovered, Theorem~\ref{thm:feedback} gives
\[
Q_t(k)
\ge
b_k.
\]
Hence the conditional probability that \(k\) is generated at least once in
the batch at iteration \(t\) is at least
\[
1-(1-b_k)^N.
\]
Since a generated occurrence of \(k\) is accepted with conditional
probability at least \(r_{\min}\), the conditional probability that \(k\) is
discovered at iteration \(t\), given that it has not previously been
discovered, is at least
\[
p_k
=
r_{\min}
\left[
1-(1-b_k)^N
\right].
\]
Sequential conditioning therefore gives
\[
\Pr(k\notin\K_T^+)
\le
(1-p_k)^T
\le
e^{-Tp_k},
\]
and hence
\[
\Pr(k\in\K_T^+)
\ge
1-e^{-Tp_k}.
\]

For every \(x\ge0\),
\[
1-e^{-x}
\ge
(1-e^{-1})\min\{1,x\}.
\]
Also,
\[
1-(1-b_k)^N
\ge
1-e^{-Nb_k}
\ge
(1-e^{-1})\min\{1,Nb_k\}.
\]
Therefore, for a constant \(c_->0\) depending only on
\(r_{\min}\) and \(C\),
\[
\Pr(k\in\K_T^+)
\ge
c_-\min\{1,NTB(k)\}.
\]

For the upper bound, Theorem~\ref{thm:feedback} gives
\[
Q_t(k)
\le
e^CB(k).
\]
Each candidate slot generates and accepts \(k\) with conditional probability
at most
\[
r_{\max}e^CB(k).
\]
A union bound over the \(NT\) candidate slots therefore yields
\[
\Pr(k\in\K_T^+)
\le
\min\left\{
1,
r_{\max}e^CNTB(k)
\right\}.
\]
Thus there exist constants \(c_-,c_+>0\), independent of \(k\), \(N\), and
\(T\), such that
\[
c_-
\sum_{k\in\K}
\min\{1,NTB(k)\}
\le
\mathbb E[D_T]
\le
c_+
\sum_{k\in\K}
\min\{1,NTB(k)\}.
\]

Since
\[
B(k)
=
B(\K)B_{\K}(k)
\]
and
\[
B_{\K}(k_j)
\asymp
j^{-\alpha},
\]
the ranked probabilities \(B(k_j)\) have the same Zipf exponent. Applying
Proposition~\ref{prop:zipf-decay} gives
\[
\mathbb E[D_T]
=
\Theta\!\left((NT)^{1/\alpha}\right)
\]
in the pre-saturation regime.

Inverting this discovery-rate relation at the expectation level gives
\[
R_{\rm cum}(D)
=
\Theta(c_{\rm gen}D^\alpha).
\]
This is an expectation-level inversion of the discovery-rate law, rather than
a claim about the expected hitting time of the \(D\)-th discovery.
\end{proof}

\section{Contamination Analysis Proof}
\label{app:contamination-proofs}

\begin{proof}[Proof of Proposition~\ref{thm:contamination}]
Condition on $\mathcal F_t$. For each undiscovered valid artifact
$k\in K\setminus K_t^+$, a single generated candidate equals $k$ and is accepted
with probability $r_{t,k}Q_t(k)$. Since the $N$ candidates and verifier decisions
are conditionally independent, the probability that $k$ is accepted at least once
during iteration $t$ is
\[
1-(1-r_{t,k}Q_t(k))^N .
\]
Therefore, by linearity of expectation,
\[
\mathbb E[\Delta G_t\mid \mathcal F_t]
=
\sum_{k\in K\setminus K_t^+}
\left(1-(1-r_{t,k}Q_t(k))^N\right).
\]

For invalid candidates, each draw falls in $\mathcal X\setminus K$ with probability
\[
U_t=\sum_{x\in\mathcal X\setminus K}Q_t(x).
\]
Conditional on being invalid, it is falsely accepted with probability $\delta_t$. Hence each
draw contributes an accepted invalid candidate with probability $\delta_t U_t$, and
linearity of expectation gives
\[
\mathbb E[\Delta I_t\mid \mathcal F_t]
=
N\delta_t U_t .
\]

For the sparse-regime expansion, use
\[
1-(1-r_{t,k}q)^N
=
Nr_{t,k}q+O(N^2q^2),
\]
where $0\le r_{t,k}\le 1$, uniformly for small $Nq$. Substituting $q=Q_t(k)$ and summing over
$k\in K\setminus K_t^+$ gives
\[
\mathbb E[\Delta G_t\mid \mathcal F_t]
=
N
\sum_{k\in K\setminus K_t^+} r_{t,k}Q_t(k)
+
O\!\left(
N^2
\sum_{k\in K\setminus K_t^+}Q_t(k)^2
\right).
\]
When $r_{t,k}=r_t$ over the relevant undiscovered region, this becomes
\[
\mathbb E[\Delta G_t\mid \mathcal F_t]
=
N r_t
\sum_{k\in K\setminus K_t^+}Q_t(k)
+
O\!\left(
N^2
\sum_{k\in K\setminus K_t^+}Q_t(k)^2
\right).
\]
Since
\[
M_t^{\mathrm{new}}
=
\sum_{k\in K\setminus K_t^+}Q_t(k),
\]
we obtain
\[
\mathbb E[\Delta G_t\mid \mathcal F_t]
=
N r_t M_t^{\mathrm{new}}
+
O\!\left(
N^2
\sum_{k\in K\setminus K_t^+}Q_t(k)^2
\right).
\]
Dividing the first-order expressions for $\mathbb E[\Delta I_t\mid \mathcal F_t]$ and
$\mathbb E[\Delta G_t\mid \mathcal F_t]$ yields
\[
\frac{\mathbb E[\Delta I_t\mid \mathcal F_t]}
{\mathbb E[\Delta G_t\mid \mathcal F_t]}
\approx
\frac{\delta_t U_t}{r_tM_t^{\mathrm{new}}}.
\]
\end{proof}

\paragraph{Deduplicated invalid artifacts.}
The main text counts invalid false positives per accepted candidate. If instead the retained
set deduplicates invalid artifacts, then the expected number of distinct invalid artifacts
falsely accepted at iteration $t$ is
\[
\mathbb E[\Delta I_t^{\rm dedup}\mid \mathcal F_t]
=
\sum_{x\in\mathcal X\setminus K}
\left(1-(1-\delta_t Q_t(x))^N\right).
\]
In the sparse regime,
\[
1-(1-\delta_t Q_t(x))^N
=
N\delta_t Q_t(x)+O(N^2Q_t(x)^2),
\]
where $0\le \delta_t\le 1$. Therefore,
\[
\mathbb E[\Delta I_t^{\rm dedup}\mid \mathcal F_t]
=
N\delta_t U_t
+
O\!\left(
N^2
\sum_{x\in\mathcal X\setminus K}Q_t(x)^2
\right).
\]
Thus the deduplicated and per-candidate definitions have the same first-order
sparse-regime limit, and the contamination ratio in the main text is unchanged to first
order.

\section{Verification Cost Analysis}
\label{app:verification}

Section~\ref{sec:imperfectVerification} analyzes how imperfect verification affects the
composition of newly accepted artifacts. This appendix adds a simple cost layer to that
analysis. The goal is not to derive a universal optimal allocation rule, since such a rule
would require a model of candidate ranking, verifier accuracy as a function of compute,
and domain-specific verification costs. Instead, we record first-order consequences of
finite verification budgets and show how they interact with the contamination threshold.

Recall that, at iteration \(t\), the verifier has true-positive rate \(r_t\) on new valid
artifacts and false-positive rate \(\delta_t\) on invalid candidates:
\[
r_t=\Pr[V(C_t)=1\mid C_t\in \K\setminus\K_t^+],
\qquad
\delta_t=\Pr[V(C_t)=1\mid C_t\in \X\setminus\K].
\]
Let \(\tau(c)\) denote the cost of verifying candidate \(c\), and let
\(\bar{\tau}=\mathbb{E}[\tau(C_t)]\) be the average verification cost under the current
generation distribution. For example, one may use the stylized model
\[
\tau(c)=\tau_0 \ell(c)^\beta,\qquad \beta\ge 1,
\]
where \(\ell(c)\) is the candidate length, \(\tau_0\) is a base verification cost, and
\(\beta\) captures possible superlinear scaling of verification effort with length.

\begin{proposition}[Feasible batch size under fixed verification cost]
\label{prop:feasible-batch}
Suppose each generated candidate is verified, generation costs \(c_{\mathrm{gen}}\) per
candidate, and verification costs \(\bar{\tau}\) per candidate on average. Under a per-iteration
budget \(B\), the feasible batch size is
\[
N^*=\left\lfloor \frac{B}{c_{\mathrm{gen}}+\bar{\tau}}\right\rfloor .
\]
In the sparse regime, the expected number of new genuine discoveries scales as
\[
\mathbb{E}[|S_t|]
\approx
\frac{B}{c_{\mathrm{gen}}+\bar{\tau}}\,
r_t M_t^{\mathrm{new}} .
\]
\end{proposition}

\emph{Proof sketch.}
Each candidate costs \(c_{\mathrm{gen}}+\bar{\tau}\) in expectation, so the largest feasible
batch size is \(N^*=\lfloor B/(c_{\mathrm{gen}}+\bar{\tau})\rfloor\). Substituting this into
the sparse-regime one-step discovery approximation
\(\mathbb{E}[|S_t|]\approx N r_tM_t^{\mathrm{new}}\) gives the result. \qed

\begin{remark}[No universal generation--verification split]
The formula above assumes that every generated candidate is verified. If instead the system
generates many candidates and sends only \(m\le N\) of them to verification, the optimal
split depends on the ranking or filtering model: in particular, on how the probability of
being valid changes with candidate rank. Without such a model, there is no universal
generation--verification allocation formula.
\end{remark}

\begin{corollary}[Verification-dominated regime]
\label{cor:verification-dominated}
If \(\bar{\tau}\gg c_{\mathrm{gen}}\), then
\[
N^*\approx \frac{B}{\bar{\tau}},
\qquad
\mathbb{E}[|S_t|]\approx
\frac{B}{\bar{\tau}}\,r_tM_t^{\mathrm{new}} .
\]
Thus, when verification is much more expensive than generation, the one-step discovery
rate is limited primarily by verification throughput rather than by generation cost.
\end{corollary}

\subsection{Cost-Dependent Verification}

The previous calculation treats verifier quality as fixed. In many settings, however,
additional verification compute can reduce the false-positive rate. To capture this tradeoff,
consider the stylized model
\[
\delta(w)=\delta_0\left(\frac{w}{w_0}\right)^{-a},
\qquad a>0,
\]
where \(w\) is verification compute per candidate, \(w_0\) is a reference compute level,
and \(a\) measures how quickly false positives decrease with additional verification effort.

A useful first-order proxy is the cost per reliable marginal discovery. In the sparse regime,
new genuine discoveries scale as \(r_tM_t^{\mathrm{new}}\), while false positives scale as
\(\delta(w)U_t\). If false positives are treated as a penalty against reliable progress, one
is led to objectives of the form
\[
\frac{c_{\mathrm{gen}}+w}
{r_tM_t^{\mathrm{new}}-\delta(w)U_t},
\]
up to problem-dependent constants. This proxy is intentionally stylized, but it captures the
basic tension: increasing \(w\) reduces contamination, while decreasing \(w\) increases the
number of candidates that can be processed.

\begin{proposition}[Stylized verification allocation]
\label{prop:stylized-verification-allocation}
Under the cost-dependent false-positive model
\[
\delta(w)=\delta_0\left(\frac{w}{w_0}\right)^{-a},
\]
and in the generation-dominated approximation \(w\ll c_{\mathrm{gen}}\), the stationary
verification effort for the first-order proxy above scales as
\[
w^*
=
\left(
\frac{
a\,\delta_0\,U_t\,w_0^a\,c_{\mathrm{gen}}
}{
r_tM_t^{\mathrm{new}}
}
\right)^{1/(a+1)} .
\]
\end{proposition}

\emph{Interpretation.}
The expression is not meant as a universal optimizer. Rather, it shows the direction of the
tradeoff. Verification effort should increase when the invalid mass \(U_t\) is large, when the
baseline false-positive rate \(\delta_0\) is large, or when the new-valid mass
\(M_t^{\mathrm{new}}\) is small. In particular, as \(M_t^{\mathrm{new}}\to 0\), the formula
predicts increasing verification effort, matching the contamination-trap intuition from
Section~\ref{sec:imperfectVerification}.

\subsection{Connection to the Contamination Threshold}

The main text defines the marginal contamination fraction
\[
f_t^{\mathrm{marg}}
=
\frac{\mathbb{E}[\Delta I_t\mid\mathcal{F}_t]}
{\mathbb{E}[\Delta G_t\mid\mathcal{F}_t]+\mathbb{E}[\Delta I_t\mid\mathcal{F}_t]}.
\]
In the sparse regime,
\[
f_t^{\mathrm{marg}}
\approx
\frac{\delta_t U_t}{r_tM_t^{\mathrm{new}}+\delta_tU_t}.
\]
Therefore, to keep \(f_t^{\mathrm{marg}}\le f_{\mathrm{critical}}\), it is sufficient that
\[
\delta_t
\le
\delta_t^*
=
\frac{
r_tM_t^{\mathrm{new}} f_{\mathrm{critical}}
}{
U_t(1-f_{\mathrm{critical}})
}.
\]

This threshold partitions the local behavior into three regimes:

\begin{enumerate}[leftmargin=*,itemsep=1pt]
\item \textbf{Safe discovery:}
\(\delta_t<\delta_t^*\) and \(M_t^{\mathrm{new}}\) is not too small. Newly accepted artifacts
are dominated by genuine discoveries.

\item \textbf{Contamination-limited discovery:}
\(\delta_t<\delta_t^*\), but \(M_t^{\mathrm{new}}\) is small. Genuine discoveries are rare, so
the system becomes increasingly sensitive to false positives.

\item \textbf{Local contamination collapse:}
\(\delta_t\ge \delta_t^*\). Invalid accepted artifacts exceed the allowed marginal
contamination fraction, so additional accepted data can degrade the retained set unless
verification is improved or invalid mass \(U_t\) is reduced.
\end{enumerate}

As \(M_t^{\mathrm{new}}\to0\), the safe threshold \(\delta_t^*\to0\) unless \(U_t\) shrinks
comparably. Thus the critical issue is not a fixed nonzero verification threshold, but the
collapse of the tolerated false-positive rate near the discovery frontier. Verification cost
matters because maintaining \(\delta_t<\delta_t^*\) may require increasing verification
effort precisely when genuine discoveries are becoming hardest to find.

\section{Compute Growth and Sustainability}
\label{app:compute-sustainability}
\label{app:compute}

Section~\ref{sec:discovery-cost} derives the cumulative generation-cost law
\[
R_{\mathrm{cum}}(D)=\Theta(c_{\mathrm{gen}}D^\alpha)
\]
under the cumulative accepted-valid Zipf-tail assumption and the
nonvanishing total valid mass and acceptance conditions of
Theorem~\ref{thm:rate-bound}. This appendix discusses a simple consequence of
that law: whether improvements in hardware, systems efficiency, and deployment scale can
sustain a desired discovery trajectory over calendar time. The discussion is not needed for
the main NOVA guarantees; it is a scenario analysis that compares the compute demanded
by a target discovery trajectory with the compute supplied by hardware and systems
improvements.

Let \(D(t)\) be a target cumulative discovery trajectory over calendar time, and let
\(C(t)\) denote the effective generation compute available to the NOVA process by time
\(t\). This available compute may reflect hardware efficiency, systems optimization,
inference batching, deployment scale, and budget. Verification costs are not included in
\(C(t)\); accounting for them would only make the sustainability requirement stronger.

Theorem~\ref{thm:rate-bound} immediately gives the necessary
condition
\begin{equation}
\label{eq:sustainability-condition}
C(t)\gtrsim c_{\mathrm{gen}}D(t)^\alpha .
\end{equation}
Thus compute sustainability is a comparison between the desired discovery trajectory and
the available effective compute supply.

For illustration, suppose the effective compute supply grows exponentially over a limited
horizon,
\[
C(t)=C_0 2^{t/\tau},
\]
where \(\tau\) is an effective doubling time that aggregates accelerator improvements,
systems optimization, deployment scale, and available budget. Equivalently, if chip-level
efficiency and deployed compute scale contribute multiplicatively, one may write
\[
C(t)=\eta_0 B_0 2^{t/\tau_{\mathrm{chip}}}2^{t/\tau_{\mathrm{scale}}},
\qquad
\frac{1}{\tau}=\frac{1}{\tau_{\mathrm{chip}}}+\frac{1}{\tau_{\mathrm{scale}}}.
\]
The specific values of these parameters are empirical and time-dependent, so the model
should be read as illustrative rather than predictive.

Combining this supply model with~\eqref{eq:sustainability-condition} gives
\[
D(t)\lesssim
\left(\frac{C_0}{c_{\mathrm{gen}}}\right)^{1/\alpha}
2^{t/(\alpha\tau)} .
\]
Thus exponential compute growth can support an exponentially growing discovery
trajectory, but with the exponent reduced by a factor of \(\alpha\). Larger \(\alpha\)
therefore makes the same target trajectory harder to sustain. Conversely, when
\(\alpha\) is close to one, the discovery-cost curve is closer to linear, so hardware and
systems improvements translate more directly into additional discoveries.

\paragraph{Marginal sustainability.}
The cumulative condition above concerns total compute required to reach \(D(t)\). The
corresponding marginal condition follows from the Section~\ref{sec:discovery-cost}
scaling
\[
\frac{dR_{\mathrm{cum}}}{dD}
=
\Theta(c_{\mathrm{gen}}D^{\alpha-1}).
\]
Thus the incremental compute required for the next discovery grows as \(D^{\alpha-1}\).
For any fixed maximum effective compute \(M_{\max}\) available per marginal discovery,
the sustainable frontier must satisfy
\[
c_{\mathrm{gen}}D^{\alpha-1}\lesssim M_{\max},
\]
or equivalently
\[
D\lesssim
\left(\frac{M_{\max}}{c_{\mathrm{gen}}}\right)^{1/(\alpha-1)}.
\]
This marginal threshold is highly sensitive to \(\alpha\). When \(\alpha\) is close to one,
the exponent \(1/(\alpha-1)\) is large, so the marginal wall can be far away. When
\(\alpha\) is larger, marginal costs rise much more quickly.

\paragraph{Interpretation.}
The sustainability question is not simply whether compute improves over time.
It is whether effective compute supply grows fast enough relative to the
desired discovery trajectory and the cumulative-exposure tail exponent
\(\alpha\). Larger \(\alpha\) corresponds to a more concentrated exposure
distribution: high-exposure artifacts are found early, and additional
discoveries become increasingly expensive. When \(\alpha\) is closer to one,
the exposure distribution has a heavier tail, so distinct opportunities
persist longer and marginal costs grow more slowly. However, whenever \(\alpha>1\), marginal cost still grows with \(D\), so
fixed compute budgets eventually face diminishing returns.

\subsection{Illustrative Domain Regimes}

Table~\ref{tab:domains} illustrates how the cumulative and marginal cost laws vary with
the Zipf exponent. The assigned \(\alpha\) values are heuristic placeholders; they are meant
to show the qualitative dependence on tail heaviness, not to provide empirical estimates for
these domains.

\begin{table}[h]
\centering
\small
\setlength{\tabcolsep}{4pt}
\caption{Illustrative discovery-cost regimes under different Zipf exponents. The
\(\alpha\) values are heuristic and are intended only to show how NOVA's scaling laws
depend on the heaviness of the discovery tail.}
\label{tab:domains}
\begin{tabularx}{\linewidth}{@{}lcccX@{}}
\toprule
\textbf{Domain} & \(\alpha\) & \textbf{Cumulative cost} & \textbf{Marginal cost} & \textbf{Qualitative implication} \\
\midrule
Elementary math 
& \(\sim 2.0\) 
& \(\Theta(D^2)\) 
& \(\Theta(D)\) 
& The distribution is highly concentrated: many easy artifacts are found early, but the remaining tail becomes expensive quickly. \\

Competition math 
& \(\sim 1.5\) 
& \(\Theta(D^{1.5})\) 
& \(\Theta(D^{0.5})\) 
& The tail is less concentrated: progress continues beyond the easy regime, but marginal difficulty still rises noticeably. \\

Research math 
& \(\sim 1.2\) 
& \(\Theta(D^{1.2})\) 
& \(\Theta(D^{0.2})\) 
& The tail is broad: undiscovered artifacts remain accessible over a longer horizon, so diminishing returns appear gradually. \\

Open problems 
& \(\sim 1.05\) 
& \(\Theta(D^{1.05})\) 
& \(\Theta(D^{0.05})\) 
& The tail is extremely broad: the scaling penalty is mild, but constants, verification, and support barriers may dominate. \\
\bottomrule
\end{tabularx}
\end{table}

The table highlights the role of \(\alpha\) as a domain-difficulty parameter. Larger
\(\alpha\) means that probability mass is concentrated on relatively few easier artifacts, so
the system rapidly exhausts high-probability discoveries and marginal costs rise sharply.
Smaller \(\alpha>1\) corresponds to a heavier tail: undiscovered mass decays more slowly,
and marginal costs grow more mildly. Thus the key distinction is not whether diminishing
returns exist, but how quickly they appear.

\begin{remark}[Hardware growth versus discovery goals]
The condition
\[
C(t)\gtrsim c_{\mathrm{gen}}D(t)^\alpha
\]
should be interpreted as a comparison between compute supply and discovery ambition.
If the desired trajectory \(D(t)\) grows slowly, even modest improvements in hardware and
systems may be sufficient for a long period. If \(D(t)\) grows aggressively, sustaining
progress requires correspondingly rapid growth in effective compute supply. Whether
hardware trends are sufficient therefore depends jointly on the target discovery trajectory,
the effective per-candidate cost \(c_{\mathrm{gen}}\), verification overheads, constants hidden
in the \(\Theta(\cdot)\) law, and the tail exponent \(\alpha\).
\end{remark}

\section{Human--AI Collaboration Details}
\label{app:human-details}

This appendix formalizes the human-augmented NOVA loop used in
Section~\ref{sec:human}. The goal is to isolate three ways in which human
experts can amplify discovery: guidance, generation, and verification.

\begin{definition}[Human Expert]
A human expert \(H\) is characterized by four components:
\begin{itemize}
    \item a proposal distribution \(P_H\) over candidate knowledge artifacts;
    \item a generation rate \(\lambda_H\), measured in candidates per unit human time;
    \item a verification procedure \(V_H\), with acceptance rate \(\rho_{H,t}\) on valid human-generated candidates;
    \item a guidance function \(G_H\), which modifies the model distribution from \(\Qt\) to a guided distribution \(Q_t'\).
\end{itemize}
\end{definition}

\begin{definition}[Augmented NOVA Loop]
At iteration \(t\), the human-augmented NOVA loop proceeds as follows:
\begin{enumerate}
    \item \textbf{Human guidance:} the expert modifies the AI generator from \(\Qt\) to \(Q_t' = G_H(\Qt)\).
    \item \textbf{AI generation:} the AI generates \(N_{\mathrm{AI}}\) candidates from \(Q_t'\).
    \item \textbf{Human generation:} the expert generates \(N_H = \lambda_H T_{\mathrm{gen}}\) candidates from \(P_H\), where \(T_{\mathrm{gen}}\) is the human time allocated to direct generation.
    \item \textbf{Combined verification:} AI and human-generated candidates are verified using the available formal, AI-assisted, or human verification procedures.
    \item \textbf{Accumulate and retrain:} accepted candidates are added to the retained set, and the model is updated as in the standard NOVA loop.
\end{enumerate}
\end{definition}

Recall that as defined in Section~\ref{sec:human},
\[
M_t^{\mathrm{new,guided}}
=
Q_t'(\K\setminus \K_t^+),
\qquad
M_t^{\mathrm{new},H}
=
P_H(\K\setminus \K_t^+),
\]
where \(M_t^{\mathrm{new,guided}}\) is the new-valid mass under the guided AI
distribution and \(M_t^{\mathrm{new},H}\) is the new-valid mass under the human
proposal distribution. Also 
\[
r_{\mathrm{eff},t}
=
\Pr[V_{\mathrm{eff}}(C_t)=1
\mid C_t\in \K\setminus \K_t^+,\, C_t\sim Q_t'],
\]
the effective true-positive rate for new valid AI-generated candidates after
human review.

\subsection{Amplification Decomposition}

\begin{proof}[Proof of Theorem~\ref{thm:amplification}]
In the sparse regime, the autonomous baseline discovery rate is
\[
\mathbb{E}[|S_t^{\mathrm{AI}}|]
\approx
N_{\mathrm{AI}} r_t M_t^{\mathrm{new}}.
\]
After human guidance, the AI distribution changes from \(\Qt\) to \(Q_t'\), so
the new-valid mass changes from \(M_t^{\mathrm{new}}\) to
\(M_t^{\mathrm{new,guided}}\). If human review changes the effective
true-positive rate for new valid AI-generated candidates from \(r_t\) to
\(r_{\mathrm{eff},t}\), then the guided AI contribution becomes
\[
\mathbb{E}[|S_t^{\mathrm{AI,guided}}|]
\approx
N_{\mathrm{AI}} r_{\mathrm{eff},t} M_t^{\mathrm{new,guided}}.
\]
The human expert also contributes \(N_H\) candidates from \(P_H\). Neglecting
duplicate discoveries between human- and AI-generated candidates, the expected
number of accepted new valid human-generated candidates is
\[
\mathbb{E}[|S_t^{H}|]
\approx
N_H \rho_{H,t} M_t^{\mathrm{new},H}.
\]
Therefore,
\[
A_H
=
\frac{
N_{\mathrm{AI}} r_{\mathrm{eff},t} M_t^{\mathrm{new,guided}}
+
N_H \rho_{H,t} M_t^{\mathrm{new},H}
}{
N_{\mathrm{AI}} r_t M_t^{\mathrm{new}}
}.
\]
Factoring the right-hand side gives
\[
A_H
=
\frac{M_t^{\mathrm{new,guided}}}{M_t^{\mathrm{new}}}
\cdot
\frac{r_{\mathrm{eff},t}}{r_t}
\cdot
\left(
1+
\frac{N_H \rho_{H,t} M_t^{\mathrm{new},H}}
{N_{\mathrm{AI}} r_{\mathrm{eff},t} M_t^{\mathrm{new,guided}}}
\right).
\]
Thus,
\[
A_{\mathrm{guide}}
=
\frac{M_t^{\mathrm{new,guided}}}{M_t^{\mathrm{new}}},
\qquad
A_{\mathrm{verify}}
=
\frac{r_{\mathrm{eff},t}}{r_t},
\qquad
A_{\mathrm{gen}}
=
1+
\frac{N_H \rho_{H,t} M_t^{\mathrm{new},H}}
{N_{\mathrm{AI}} r_{\mathrm{eff},t} M_t^{\mathrm{new,guided}}},
\]
and hence
\[
A_H
=
A_{\mathrm{guide}}\cdot A_{\mathrm{verify}}\cdot A_{\mathrm{gen}}.
\]
The generation term is measured relative to the already guided and
human-reviewed AI discovery rate; the guidance and verification effects are
therefore accounted for separately.
\end{proof}

\subsection{Support Expansion}

\begin{proof}[Proof of Theorem~\ref{thm:human-support}]
By definition, an artifact can be generated with positive probability under a
distribution only if it lies in that distribution's support. Hence the reachable
valid set under the autonomous generator \(\Qt\) is
\(\K\cap\supp(\Qt)\), while after human guidance it becomes
\(\K\cap\supp(Q_t')\). If
\[
\K\cap\bigl(\supp(Q_t')\setminus\supp(\Qt)\bigr)\neq\emptyset,
\]
then there exists a valid artifact that is reachable under \(Q_t'\) but not under
\(\Qt\). Therefore,
\[
\K\cap\supp(\Qt)
\subsetneq
\K\cap\supp(Q_t'),
\]
so human guidance strictly expands the reachable valid set. In particular, if
\(Q_t'\) assigns non-negligible mass to valid artifacts outside the previous
effective support, those artifacts are no longer blocked by the autonomous
exploration barrier.
\end{proof}

\begin{remark}[Reachability versus discovery]
Theorem~\ref{thm:human-support} expands the set of potentially reachable valid
artifacts, but does not by itself guarantee their eventual discovery. Eventual
discovery still requires sufficient probability mass on the newly reachable
artifacts and nondegenerate acceptance by verification, as in
Theorem~\ref{thm:convergence}.
\end{remark}

\subsection{Human Effort Allocation}

\begin{proposition}[Human Effort Allocation Principle]
\label{prop:human-alloc}
Given limited human time \(T_H\), the appropriate use of human effort depends on
the system's bottleneck:
\begin{enumerate}
    \item \textbf{Far from the exploration barrier} \((M_t^{\mathrm{new}}\) large): prioritize human generation, because many valid artifacts remain reachable.
    \item \textbf{Near the exploration barrier} \((M_t^{\mathrm{new}}\) small): prioritize guidance, because redirecting or expanding support is more valuable than generating more samples from a depleted region.
    \item \textbf{When formal verification is unavailable} \((r_t \approx 0)\): prioritize human verification, because valid artifacts may be generated but cannot be reliably accumulated.
\end{enumerate}
\end{proposition}

\begin{remark}
Proposition~\ref{prop:human-alloc} is a design principle rather than a formal
optimization theorem. A rigorous optimum would require a concrete objective
function and a cost model for allocating human effort among guidance, generation,
and verification.
\end{remark}

\section{Infinite Knowledge Spaces}
\label{app:infinite}

The main convergence theorem (Theorem~\ref{thm:convergence}) assumes \(|\K|<\infty\), so artifact-wise discovery implies eventual coverage of the whole knowledge domain. Many natural domains, however, are effectively infinite: valid theorems, algorithms, programs, designs, and scientific hypotheses do not come from a small finite catalog. This appendix records how the NOVA conclusions should be interpreted in countably infinite spaces.

\begin{definition}[Infinite NOVA]
An infinite NOVA system operates over a countably infinite knowledge space
\(\K=\{k_1,k_2,\ldots\}\), with an ideal difficulty distribution \(P\) satisfying
\(\sum_{j\ge 1}P(k_j)=1\).
\end{definition}

In this setting, the almost-sure coverage conclusion of Theorem~\ref{thm:convergence} must be interpreted artifact-wise. If each fixed \(k\in\K\) receives divergent aggregate pre-discovery exposure and nondegenerate acceptance, then each such \(k\) is eventually discovered almost surely. However, no finite time can cover an infinite \(\K\), and global coverage must be replaced by growth of the discovered set and decay of the remaining mass.

\begin{remark}[Infinite support and the exploration barrier]
The exploration barrier remains an effective-support statement. Even if \(Q_0\) has full literal support on \(\K\), artifacts assigned vanishingly small probability may be unreachable under any finite compute budget. Thus infinite support alone does not eliminate the need for support expansion, probability-mass redirection, or human guidance.
\end{remark}

\begin{remark}[Zipf exponents near the boundary]
For a countably infinite Zipf law \(P(k_j)\propto j^{-\alpha}\), normalizability requires \(\alpha>1\). At the boundary \(\alpha=1\), the pure Zipf law is not normalizable since \(\sum_{j\ge1}j^{-1}=\infty\); studying this case requires truncation or another slowly varying heavy-tail model. For \(0<\alpha<1\), the infinite pure power law is also improper, though finite truncations can still approximate extremely heavy-tailed pre-saturation behavior.
\end{remark}

\begin{corollary}[Infinite knowledge horizon]
\label{cor:infinite-horizon}
In a countably infinite knowledge space, complete discovery is asymptotic
rather than finite-time. Under the cumulative accepted-valid Zipf-tail
condition of Assumption~\ref{asm:tail-equiv},
\[
\mathbb E[D_T]
=
\Theta\!\left(E_T^{1/\alpha}\right)
\]
for all horizons over which the tail bounds in
Assumption~\ref{asm:tail-equiv} apply.

Under anchored recursive feedback with
\[
B_{\K}(k_j)
\asymp
j^{-\alpha},
\qquad
\alpha>1,
\]
and acceptance probabilities bounded below by a positive constant, this
specializes to
\[
\mathbb E[D_T]
=
\Theta\!\left((NT)^{1/\alpha}\right).
\]
No finite time can cover all of a countably infinite \(\K\).
\end{corollary}

\begin{remark}[Finite versus infinite regimes]
For a large finite domain, the relevant behavior depends on the cumulative
repeat-inclusive accepted valid exposure
$
E_T
=
\sum_{t=0}^{T-1}
Nr_tM_t^{\mathrm{val}}.
$
In the pre-saturation regime, the process behaves like infinite-tail
occupancy sampling. Near saturation, finite-size effects dominate. Once the
reachable finite domain is saturated, further generation mostly repeats known
artifacts unless the support or cumulative exposure tail changes. Most
practical frontier-discovery settings are best viewed as pre-saturation or
support-limited rather than near full coverage.
\end{remark}

\clearpage

\end{document}